%% file: interpolate-main.tex
\newif\iffull
\fulltrue

\documentclass[11pt,twoside]{article}
\usepackage{xspace,fullpage,amsmath,amsthm,amssymb,amsfonts,boxedminipage,microtype,hyperref}
\usepackage{paralist}
\usepackage{bm}
\usepackage{bbm}
\usepackage{array}
\usepackage{multirow}
\usepackage{times}
\usepackage{graphicx}

\usepackage{color}
\usepackage[style=alphabetic,backend=bibtex,maxbibnames=20,maxcitenames=6,giveninits=true,doi=false,url=true]{biblatex}
\newcommand*{\citet}[1]{\AtNextCite{\AtEachCitekey{\defcounter{maxnames}{2}}} \textcite{#1}}

\newcommand*{\citep}[1]{\cite{#1}}

\bibliography{vf-allrefs-local,interpolate}

\usepackage{vfmacros}

\newcommand{\reals}{{\mathbb R}}
\newcommand{\ind}[1]{{\mathbf 1} \lp #1 \rp}

\newcommand{\err}{\mathtt{err}}
\newcommand{\mem}{\mathtt{mem}}
\newcommand{\conf}{\mathtt{conf}}
\newcommand{\errn}{\mathtt{errn}}

\newcommand{\single}{\mathtt{single}}
\newcommand{\spn}{\mathtt{span}}
\newcommand{\opt}{\mathtt{opt}}
\newcommand{\aerr}{\overline{\err}}

\newcommand{\weight}{\mathtt{weight}}
\providecommand{\cG}{{\mathcal G}}
\providecommand{\dfn}{:=}
\newcommand{\piN}{\bar{\pi}^N}
\newcommand{\prior}{\pi}

\newcommand{\Dist}{\mathop{\mathbf D}}
\newcommand{\TV}{\mathtt{TV}}

\title{Does Learning Require Memorization?\\ A Short Tale about a Long Tail}

\author{Vitaly Feldman\thanks{Now at Apple. Part of this work was done while the author was visiting the Simons Institute for the Theory of Computing.} \\ Google Research, Brain Team }

\begin{document}
\date{}
\maketitle

\begin{abstract}%
State-of-the-art results on image recognition tasks are achieved using over-parameterized learning algorithms that (nearly) perfectly fit the training set and are known to fit well even random labels. This tendency to memorize the labels of the training data is not explained by existing theoretical analyses. Memorization of the training data also presents significant privacy risks when the training data contains sensitive personal information and thus it is important to understand whether such memorization is necessary for accurate learning.

We provide the first conceptual explanation and a theoretical model for this phenomenon. Specifically, we demonstrate that for natural data distributions memorization of labels is {\em necessary} for achieving close-to-optimal generalization error. Crucially, even labels of outliers and noisy labels need to be memorized. The model is motivated and supported by the results of several recent empirical works. In our model, data is sampled from a mixture of subpopulations and our results show that 
 memorization is necessary whenever the distribution of subpopulation frequencies is long-tailed. Image and text data is known to be long-tailed and therefore our results establish a formal link between these empirical phenomena. 
Our results allow to quantify the cost of limiting memorization in learning and explain the disparate effects that privacy and model compression have on different subgroups. \end{abstract}

\thispagestyle{empty}
\newpage
\setcounter{page}{1}

\input{interpolate-intro}

\input{interpolate-tails}

\input{interpolate-privacy}

\section{Discussion}
\label{sec:discuss}
Our work provides a natural and simple learning model in which memorization of labels and, in some cases interpolation, are necessary for achieving nearly optimal generalization when learning from a long-tailed data distribution. It suggests that the reason why many modern ML methods reach their best accuracy while (nearly) perfectly fitting the data is that these methods are (implicitly) tuned to handle the long tail of natural data distributions. Our model explicitly incorporates the prior distribution on the frequencies of subpopulations in the data and we argue that such modeling is necessary to avoid the disconnect between the classical view of generalization and the practice of ML. We hope that the insights derived from our approach will serve as the basis for future theoretical analyses of generalization that more faithfully reflect modern datasets and learning techniques. A recent example that such modeling has practical benefits can be found in \citep{cao2019learning}.

\subsection*{Acknowledgements}
Part of the inspiration and motivation for this work comes from empirical observations that differentially private algorithms have poor accuracy on atypical examples. I'm grateful to Nicholas Carlini, Ulfar Erlingsson and Nicolas Papernot for numerous illuminating discussions of experimental work on this topic \citep{carlini2019distribution} and to Vitaly Shmatikov for sharing his insights on this phenomenon in the context of language models. I would like to thank my great colleagues Peter Bartlett, Misha Belkin, Olivier Bousquet, Edith Cohen, Roy Frostig, Daniel Hsu, Phil Long, Yishay Mansour, Mehryar Mohri, Tomer Koren, Sasha Rakhlin, Adam Smith, Kunal Talwar, Greg Valiant, and Chiyuan Zhang for insightful feedback and suggestions on this work. I thank the authors of \citep{zhu2014capturing} for the permission to include Figure~1 from their work.

\printbibliography


\input{interpolate-app}

\end{document}

%% file: interpolate-intro.tex
\section{Introduction}
Understanding the generalization properties of learning systems based on deep neural networks (DNNs) is an area of great practical importance and significant theoretical interest.  The models used in deep learning are famously overparameterized, that is, contain many more tunable parameters than available data points. This makes it is easy to find models that ``overfit'' to the data by effectively memorizing the labels of all the training examples. The standard theoretical approach to understanding of how learning algorithms avoid such {\em overfitting} is based on the idea of regularization. Learning algorithms are designed to either explicitly or implicitly balance the level of the model's complexity (and, more generally, its ability to fit arbitrary data) and the empirical error on the training dataset. Fitting each mislabeled point or an outlier requires increasing the level of model's complexity and therefore, by tuning this balance, the learning algorithm can find the patterns in the data without overfitting.

A variety of regularization techniques are widely used in practice and have been analyzed theoretically. Yet, the accepted view of regularization contradicts the empirical evidence from most modern image and text classification datasets. Deep learning algorithms tend to produce models that fit the training data very well, typically achieving $95$-$100\%$ accuracy, even when the accuracy on the test dataset is much more modest (often in the $50$-$80\%$ range). Such (near) perfect fitting requires memorization\footnote{In this work we will formalize and quantify this notion of memorization. Informally, we say that a learning algorithm memorizes the label of some example $(x,y)$ in its dataset $S$ if the model output on $S$ predicts $y$ on $x$ whereas the model obtained by training on $S$ without $(x,y)$ is unlikely to predict $y$ on $x$.} of mislabeled data and outliers which are inevitably present in large datasets. Further, it is known that the same learning algorithms achieve training accuracy of over $90\%$ on the large ImageNet dataset \cite{imagenet_cvpr09} that is labeled completely randomly \cite{ZhangBHRV17}. It is therefore apparent that these algorithms are not using regularization that is sufficiently strong to prevent memorization of (the labels of) mislabeled examples and outliers.

This captivating disconnect between the classical theory and modern ML practice has attracted significant amount of research and broad interest in recent years (see Sec.~\ref{sec:related} for an overview). At the same time the phenomenon is far from new. Random forests \citep{breiman2001random} and Adaboost \citep{FreundSchapire:97} are known to achieve their optimal generalization error on many learning problems while fitting the training data perfectly \citep{SchapireFBL98,schapire2013explaining,wyner2017explaining}. There is also recent evidence that this holds for kernel methods in certain regimes as well \citep{ZhangBHRV17,belkin18kernel,liang2018just}.

Understanding this disconnect is also of significant importance in the context of privacy-preserving machine learning. Privacy is a natural concern when the training data contains sensitive information about individuals such as medical records or private communication. The propensity of deep learning algorithms to memorize training data is known to pose privacy risks when the resulting model is deployed \cite{ShokriSSS17}. This leads to the question of whether such memorization is necessary for learning with high accuracy or is merely an artifact of the current learning methods.

\subsection{Our contribution}
We propose a conceptually simple explanation and supporting theory for why memorization of seemingly useless labels may be necessary to achieve close-to-optimal generalization error. It is based on the view that the primary hurdle to learning an accurate model is not the noise inherent in the labels but rather an insufficient amount of data to predict accurately on rare and atypical instances. Such instances are usually referred in practice as the ``long tail'' of the data distribution. It has been widely observed that modern datasets used for visual object recognition and text labeling follow the classical long-tailed distributions such as Zipf distribution (or more general power law distributions). 

To formalize the notion of having a ``long tail'' we will model the data distribution of each class (in a multiclass prediction problem) as a mixture of distinct subpopulations.
For example, images of birds include numerous different species photographed from different perspectives and under different conditions (such as close-ups, in foliage and in the sky) \citep{van2017devil}. Naturally, the subpopulations may have different frequencies (which correspond to mixture coefficients). We model the informal notion of long-tailed data distributions as distributions in which the frequencies of subpopulations are long-tailed. The long-tailed nature of subpopulation frequencies is known in datasets for which additional human annotations are available. A detailed discussion of this phenomenon in the SUN object detection benchmark \citep{xiao2010sun} can be found in the work of  \citet{zhu2014capturing}. In Fig.~\ref{fig:long} we include a plot from the work that demonstrates the long tail of the frequency distribution.

\begin{figure}
  {\centering\includegraphics[width=0.8\linewidth]{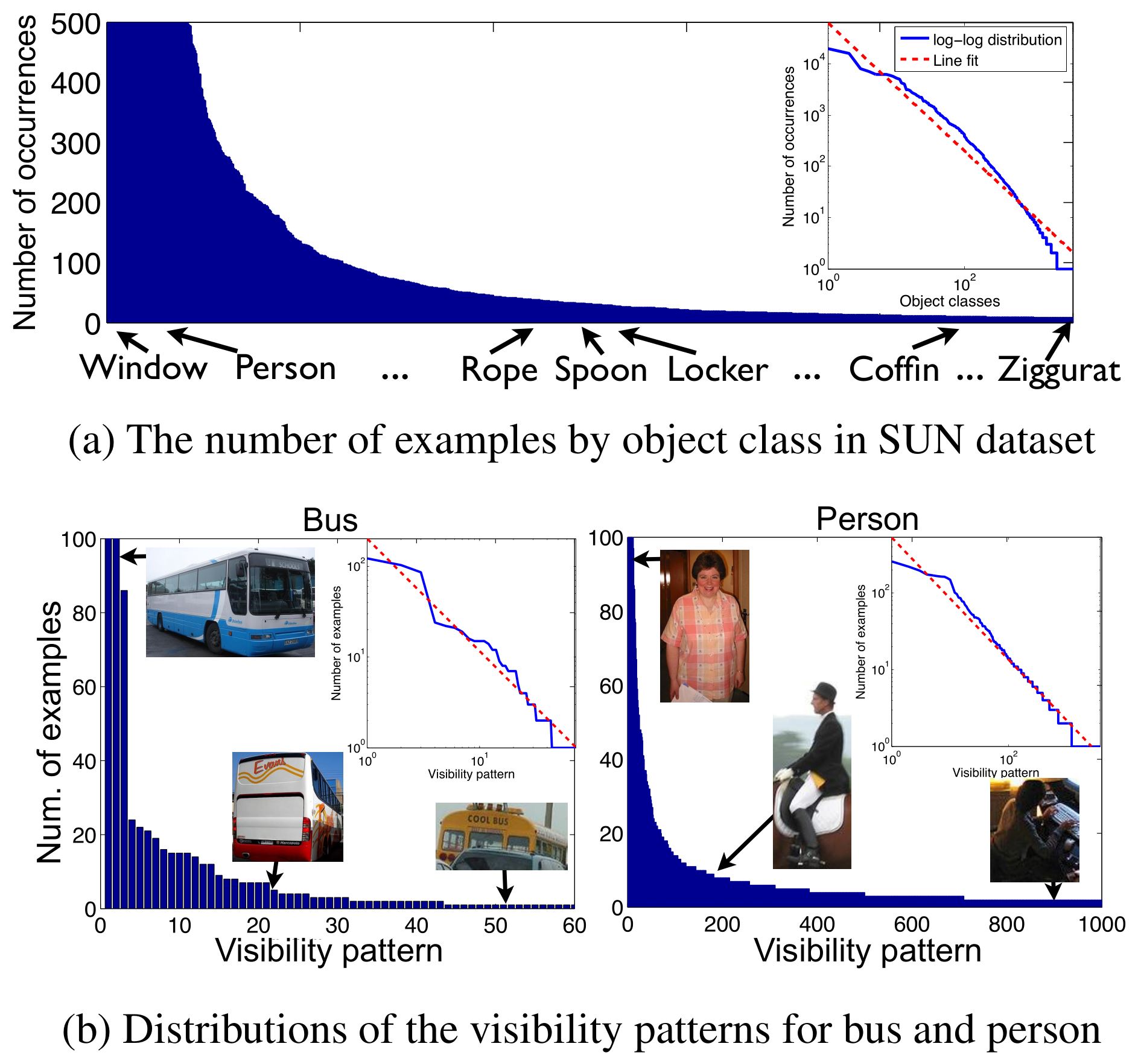}}
  \caption{Long tail of class frequencies and subpopulation frequencies within classes. The figure is taken from \citep{zhu2014capturing} with the authors' permission.}
  \label{fig:long}
\end{figure}

Additional evidence that classes can be viewed as long-tailed mixtures of subpopulations comes from extreme multiclass problems. Specifically, these problems often have more than $10,000$ fine-grained labels and the number of examples per class is long-tailed \citep{babbar2017dismec,wang2017learning,krishna2017visual,van2017devil,Cui_2018_CVPR,Babbar2019}. Observe that fine-grained labels in such problems correspond to subcategories of coarser classes (for example, different species of birds all correspond to the ``bird" label in a coarse classification problem). We also remark that subpopulations do not have to directly correspond to human-definable categories. They are the artifacts of the representation used by the learning algorithm which are often relatively low-level.

It is natural to presume that before seeing the dataset the learning algorithm does not know the frequencies of subpopulations. The second key observation underlying our explanation is that the algorithm may not be able to predict accurately on a subpopulation until at least one example from the subpopulations is observed. Alternatively, the accuracy of the algorithm on a subpopulation is likely to increase noticeably once a representative example from that subpopulation is observed. A dataset of $n$ samples from a long-tailed mixture distribution will have some subpopulations from which just a single example was observed (and some subpopulations from which none at all). To predict more accurately on a subpopulation from which only a single example was observed (and to fit the example) the learning algorithm needs to memorize the label of the example. The question is whether this is necessary for achieving close-to-optimal generalization error. The answer depends on the frequency of the subpopulation. If the unique example from a subpopulation (or {\em singleton}) comes from an extremely rare (or ``outlier'') subpopulation then memorizing it has no significant benefits. At the same time, if the singleton comes from an ``atypical" subpopulation with frequency on the order of $1/n$, then memorizing such an example is likely to improve the accuracy on the entire subpopulation and thereby reduce the generalization error by $\Omega(1/n)$.

The key point of this work is that based on observing a single sample from a subpopulation, it is impossible to distinguish samples from ``atypical'' subpopulations from those in the ``outlier'' ones.  Therefore an algorithm can only avoid the risk of missing ``atypical'' subpopulations by also memorizing the labels of singletons from the ``outlier'' subpopulations. Importantly, in a long-tailed distribution of frequencies, the total weight of frequencies on the order of $1/n$ is significant enough that ignoring these subpopulations will hurt the generalization error substantially. Thus, for such distributions, an algorithm needs to memorize the labels of outliers in order to achieve close-to-optimal generalization.

The long tail effect also explains why memorizing mislabeled examples can be necessary. As discussed, a learning algorithm may be unable to infer the label of a singleton example accurately based on the rest of the dataset. Thus as long as the observed label is the most likely to be true and the singleton comes from an ``atypical'' subpopulation, the algorithm needs to memorize the label. In contrast, if the mislabeled example comes from a subpopulation with many other examples in the dataset, the correct label can be inferred from the other labels and thus memorization is not necessary (and can even be harmful). In most datasets used in machine learning benchmarks only relatively atypical examples are mislabeled and the noise rate is low. Thus learning algorithms for such datasets are tuned to memorize the labels quite aggressively.

\subsubsection{Overview}
On a technical level our primary contribution is turning this intuitive but informal explanation into a formal model that allows to quantify the trade-offs involved. This model also allows to quantify the cost of limiting memorization (for example, via regularization or ensuring differential privacy) when learning from natural data distributions.

We start by explaining why achieving close-to-optimal generalization error requires fitting outliers and (some) mislabeled examples since this is the phenomenon observed in practice. We then formalize the claim that such fitting requires label memorization. Our explanation is based on a simple model for classification problems that incorporates the long tail of frequencies in the data distribution. The goal of the model is to isolate the discussion of the effect of memorization on the accuracy from other aspects of modeling subpopulations. 
More formally, in our model the domain $X$ is unstructured and has size $N$ (each point will correspond to a subpopulation in the more general model). In the base model the true labeling function belongs to some class of functions $F$ known to the learning algorithm. We will be primarily interested in the setting where $F$ is rich (or computationally hard) enough that for a significant fraction of the points the learning algorithm cannot predict the label of a point well without observing it in the dataset. In particular, fitting some of the examples will require memorizing their labels.

Nothing is known a priori about the frequency of any individual point aside from a prior distribution over the frequencies described by a list of $N$ frequencies $\prior=(\prior_1,\ldots,\prior_N)$. Our results are easiest to express when the objective of the learning algorithm is to minimize the expectation of the error over a random choice of the marginal distribution $D$ over $X$ from some meta-distribution $\D$ (instead of the more usual worst-case error). In addition, for convenience of notation we will also measure the error with respect to a random choice of the labeling function from some distribution $\F$ over $F$. That is, the objective of a learning algorithm $\A$ is defined as:
$$ \aerr(\D,\F,\A) \dfn \E_{D\sim \D, f\sim \F} \lb \E_{S \sim (D,f)^n,\ h\sim \A(S)} \lb  \pr_{x\sim D}[h(x) \neq f(x)] \rb \rb .$$

Specifically, we consider the following meta-distribution over marginal distributions on $X$: the frequency of each point in the domain is chosen randomly and independently from the prior $\prior$ of individual frequencies and then normalized to 1. This process results in a meta-distribution $\D$ over marginal distributions that is similar to choosing the frequencies of the elements to be a random permutation of the elements of $\prior$. Models measuring the worst-case error over all the permutations of a list of frequencies underlie the recent breakthroughs in the analysis of density estimation algorithms \citep{OrlitskyT:15,ValiantV:16}. We believe that results similar to ours can be obtained in this worst-case model as well and leave such an extension for future work\footnote{The extension to measuring the worst-case error over the choice of $f\in F$, on the other hand, is straightforward.}.

Our main result (Thm.~\ref{thm:main-bound}) directly relates the number of points that an algorithm does not fit to the sub-optimality (or excess error) of the algorithm via a quantity that depends only on the frequency prior $\pi$ and $n$. Importantly, excess error is measured relative to the optimal algorithm and not relative to the best model in some class. Formally, we denote by $\errn_S(\A,1)$ the number of examples that appear once in the dataset $S$ and are mislabeled by the classifier that $\A$ outputs on $S$. A special case of our theorem states:
\equ{\aerr(\prior,\F,\A) \geq \opt(\prior,\F) + \tau_1 \cdot \E\lb \errn_S(\A,1) \rb. \label{eq:rel-intro}}
Here $\aerr(\prior,\F,\A)$ refers to the expected generalization error of $\A$ and $\opt(\prior,\F)$ is the minimum achievable error by any algorithm (expectations are with respect to the meta-distribution over learning problems resulting from the process we described, randomness of the learning algorithm and also sampling of the dataset).
The important quantity here is
$$\tau_1 \dfn  \frac{\E_{\alpha \sim \piN} \lb \alpha^2 \cdot (1-\alpha)^{n-1}\rb}{\E_{\alpha \sim \piN} \lb \alpha \cdot (1-\alpha)^{n-1}\rb } ,$$
where $\piN$ is the actual marginal distribution over frequencies that results from our process and is, basically, a slightly smoothed version of $\pi$. We note that the optimal algorithm in this case does not depend on $\prior$ and thus our modeling does not require the learning algorithm to know $\pi$ to achieve near-optimal generalization error.

The quantity $\tau_1$ is easy to compute given $\pi$. As a quick numerical example, for the prototypical long-tailed Zipf distribution (where the frequency of the $i$-th most frequent item is proportional to $1/i$) over the universe of size $N=50,000$ and $n=50,000$ samples, one gets the expected loss of at least $\approx 0.47/n$ per {\em every} example the learner does not fit. For comparison, the worst-case loss (per point) in this setting is determined by the least frequent element and is $\approx 0.09/n$. Given that the expected fraction of samples that appear once is $\approx 17\%$, an algorithm that does not fit well will be suboptimal by $\approx 7\%$ (with the optimal top-$1$ error for $10$ balanced classes being $\approx 15\%$ in this case). More generally, we show that $\tau_1$ can be lower bounded by the total weight of the part of the prior $\pi$ which has frequency on the order of $1/n$ and also that the absence of frequencies on this order will imply negligible $\tau_1$ (see Sec.~\ref{sec:tail-bounds} for more details).

In our basic model the data is labeled correctly and fitting all the training examples (also referred to as {\em interpolation}) is the optimal strategy. We extend our model to a more general setting in which examples can be mislabeled. Under the assumption that the learning algorithm's prior makes the observed label the most likely to be correct by some margin $\kappa$ we demonstrate that memorization of labels is necessary for singleton examples. The cost of not fitting given in eq.~\eqref{eq:rel-intro} is now multiplied by $\kappa$ (see Sec.~\ref{sec:noise} for details). Note that in the presence of noise, interpolation may no longer be the optimal strategy, and in particular, memorization of noisy labels can be necessary even in the non-interpolating regime.

\paragraph{Continuous data distributions:}
Naturally, our simple setting in which individual points have significant probability does not capture the continuous and high-dimensional ML problems where each individual point has an exponentially small (in the dimension) probability. In this more general setting the prediction on the example itself has negligible effect on the generalization error. To show how the effects we demonstrated in the simple discrete setting extend to continuous distributions, we consider mixture models of subpopulations. In our model, the frequencies of subpopulations (or mixture coefficients) are selected randomly according to the prior $\prior$ as before. The labeling function is also chosen as before and is assumed to be constant over every subpopulation.

The discussion of the relationship between fitting the dataset and generalization makes sense only if one assumes that the prediction on the data point in the dataset will affect the predictions on related points. In our setting it is natural to assume that (with high probability) the learning algorithm's prediction on a single point from a subpopulation will be correlated with the prediction on a random example from the same subpopulation. We refer to this condition as coupling (Defn.~\ref{def:subpop-coupled}) and show that eq.~\eqref{eq:rel-intro} still holds up to the adjustment for the strength of the coupling.

Intuitively, it is clear that this form of ``coupling'' is likely to apply to ``local'' learning rules such as the nearest neighbors algorithm. Indeed, our assumption can be seen as a more abstract version of geometric smoothness conditions on the marginal distribution of the label used in analysis of such methods (\eg \citep{ChaudhuriD14}). We also show that it applies to linear predictors/SVMs in high dimension provided that distinct subpopulations are sufficiently uncorrelated (see Sec.~\ref{sec:examples}). Deep neural networks are known to have some of the properties of both nearest neighbor rules and linear classifiers in the last-hidden-layer representation (\eg \citep{cohen2018dnn}). Thus DNNs are likely to exhibit this type of coupling as well.

\paragraph{From fitting to memorization and privacy:} The results we described so far demonstrate that 
an algorithm that does not fit the training data well will be suboptimal on long-tailed data distributions. Fitting of training data was not previously explained only when the learning algorithm fits the training labels much better than the test data, in other words, when the generalization gap is large (often $>20\%$). Such fitting suggests that the training algorithm memorized a large fraction of the training labels. To make this intuition formal we give a simple definition of what memorizing a label of a point in the dataset means (we are not aware of a prior formal definition of this notion). Formally, for a dataset $S=(x_i,y_i)_{i\in [n]}$ and $i \in [n]$ define
 $$\mem(\A,S,i) \dfn \pr_{h\sim \A(S)}[h(x_i) = y_i] - \pr_{h\sim \A(S^{\setminus i})}[h(x_i) = y_i] ,$$
 where $S^{\setminus i}$ denotes the dataset that is $S$ with $(x_i,y_i)$ removed.
 This value is typically non-negative and we think of label as memorized when this value is larger than some fixed positive constant (such as $0.5$). Namely, the label of an example is memorized if it is fit well by the algorithm despite being hard to predict based on the rest of the dataset.

This definition is closely related to the classical leave-one-out notion of stability \citep{DevroyeW79,BousquettE02} but focuses on the change in the label and not in the incurred loss. As in the case of stability, our notion of label memorization is directly related to the expected generalization gap. Indeed, the expectation over the choice of dataset of the average memorization value is equal to the expectation of the generalization gap. Thus a large generalization gap implies that a significant fraction of labels is memorized. 

An immediate corollary of this definition is that an algorithm with a limited ability to memorize labels will not fit the singleton data points well whenever the algorithm cannot predict their labels based on the rest of the dataset. Two natural situations in which the algorithm will not be able to predict these labels are learning a complex labeling function (\eg having large VC dimension) and computational hardness of finding a simple model of the data. In addition, the labels are also hard to predict in the presence of noise. A direct corollary of our results is that limiting memorization (for example via regularization or model compression) and differential privacy has costs in terms of achievable generalization error. The sharp quantitative nature of these results allows us to explain recent empirical findings demonstrating that these costs can disproportionably higher for less frequent subgroups in the population (see Section~\ref{sec:disparate} for details).

\subsection{Known empirical evidence}
The best results (that we are aware of) on modern benchmarks that are achieved without interpolation are those for differentially private (DP) training algorithms \citep{abadi2016deep,PapernotAEGT16,PapernotAEGT17,McMahan18}. While not interpolating is not the goal, the properties of DP imply that a DP algorithm with the privacy parameter $\eps = O(1)$ cannot memorize individual labels \iffull (see Sec.\ref{sec:privacy} for more details on why)\fi. Moreover, they result in remarkably low gap between the training and test error that is formally explained by the generalization properties of DP \citep{DworkFHPRR14:arxiv}. However, the test error results achieved in these works are well below the state-of-the-art using similar models and training algorithms. For example, \citet{PapernotAEGT17} report accuracy of $98\%$ and $82.7\%$ on MNIST and SVHN as opposed to $99.2\%$ and $92.8\%$, respectively when training the same models without privacy. 

The motivation and inspiration for this work comes in part from attempts to understand why DP algorithms fall short of their non-private counterparts and which examples are they more likely to misclassify. A thorough and recent exploration related to this question can be found in the work of \citet{carlini2019distribution}. They consider different ways to measure how ``prototypical'' each of the data points is according to several natural metrics and across MNIST, CIFAR-10, Fashion-MNIST and ImageNet datasets and compare between these metrics. One of those metrics is the highest level of privacy that a DP training algorithm can achieve while still correctly classifying an example that is correctly classified by a non-private model. As argued in that work and is clear from their comprehensive visualization, the examples on which a DP model errs are either outliers or atypical ones.
 To illustrate this point, we include the examples for MNIST digit ``3'' and CIFAR-10 ``plane'' class from their work as Fig.~\ref{fig:prototypical}.
In addition, the metric based on DP is well correlated with other metrics of being prototypical such as relative confidence of the (non-private) model and human annotation. Their concepts of most and least prototypical map naturally to the frequency of subpopulation in our model. Thus their work supports the view that the reason why learning with DP cannot achieve the same accuracy as non-private learning is that it cannot memorize the tail of the mixture distribution. This view also explains the recent empirical results showing that the decrease in accuracy is larger for less well represented subpopulations \citep{BagdasaryanShmatikov19}.

\begin{figure}
  \includegraphics[width=\linewidth]{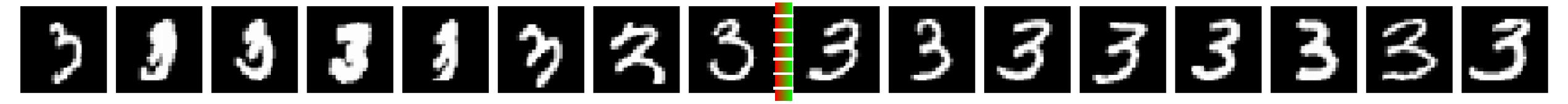}
  \includegraphics[width=\linewidth]{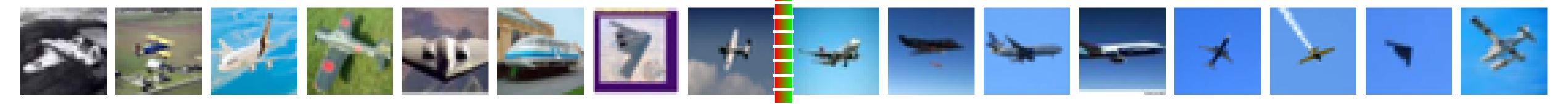}
  \caption{Hardest examples for a differentially private to predict accurately (among those accurately predicted by a non-private model) on the left vs the easiest ones on the right. Top row is for digit ``3" from the MNIST dataset and the bottom row is for the class ``plane" from the CIFAR-10 dataset. The figure is extracted from \citep{carlini2018prototypical} with the authors' permission. Details of the training process can be found in the original work.}
  \label{fig:prototypical}
\end{figure}

\iffull
Another empirical work that provides indirect support for our theory is \citep{arpit2017closer}. It examines the relationship between memorization of random labels and performance of the network for different types of regularization techniques. The work demonstrates that for some regularization techniques it is possible to reduce the ability of the network to fit random labels without significantly impacting their performance on true labels. The explanation proposed for this finding is that memorization is not necessary for learning. While it may appear to contradict our theory, a closer look at the result suggests the opposite conclusion. On the true labels almost all their regularization techniques still reach near perfect train accuracy with test accuracy of at most $78\%$. The only two techniques that do not quite interpolate (though still reaching around $97\%$ train accuracy) are {\em exactly} the ones that do exhibit clear correlation between ability to fit random labels and test accuracy (see ``input binary mask'' and ``input gaussian'' in their Figs.~10 and 11). We remark (and elaborate in Section~\ref{sec:discuss}) that fitting random examples or even interpolation are not necessary conditions for the application of our approach and for memorization being beneficial.
\fi

In a subsequent work with Chiyuan Zhang \citep{FeldmanZhang20} we investigate label memorization and test the predictions of our theory directly. In particular, using an efficiently computable proxy for the memorization score, we discover examples whose labels are memorized in MNIST, CIFAR-10/100, and ImageNet datasets. Visual inspection of these examples confirms that these examples are a mix of outlier/mislabeled examples and correctly labeled but atypical examples. We then demonstrate that memorized examples are important for learning as removing them from the training set decreases the accuracy of the resulting model significantly. Further, the long-tail theory in this work predicts that there is a significant fraction of examples whose memorization is necessary for predicting accurately on examples from the same subpopulation in the test set. More formally, there exist examples in the training set such that for each of them $(1)$ the label is memorized by the learning algorithm in the sense defined above; $(2)$ there exists a dependent example in the test set in the following sense: the accuracy of the model on the dependent test example drops significantly when the corresponding example from the training set is removed (with no significant effect on the accuracy on the other test examples). We design an algorithm for testing this prediction efficiently. The results of this algorithm on MNIST, CIFAR-100, and ImageNet datasets reveal numerous visually similar pairs of relatively atypical examples \citep{FeldmanZhang20,FeldmanZhang:20site}.

\subsection{Related work}
\label{sec:related}


One line of research motivated by the empirical phenomena we discuss here studies implicit regularization in the overparameterized regime (namely, when the parameter space is large enough that the learning algorithm can perfectly fit the dataset). For example, the classical margin theory \citep{Vapnik:82,cortes1995support,SchapireFBL98} for SVMs and boosting suggests that, while the ambient dimension is large, the learning algorithm implicitly maximizes the margin. The generalization gap can then be upper bounded in terms of the margin. Examples of this approach in the context of DNNs can be found in \citep{NeyshaburTS14,neyshabur2017exploring,bartlett2017spectrally,neyshabur2017geometry,li2018algorithmic} (and references therein). These notions imply that it is beneficial to overparameterize and suffice for explaining why the training algorithm will select the best model among those that do fit the training set. However implicit regularization does not explain why, despite the regularization, the training error is near zero even when the generalization error is large.

Another line of research studies generalization properties of learning algorithms that fit the training data perfectly, often referred to as {\em interpolating} \citep{belkin2018overfitting,belkin18kernel}. For example, a classical work of \citet{cover1967nearest} gives bounds on the generalization error of the 1-nearest neighbor algorithm. Recent wave of interest in such methods has lead to new analyses of existing interpolating methods as well as new algorithmic techniques \citep{wyner2017explaining,belkin2018does,belkin2018overfitting,liang2018just,belkin18kernel,rakhlin19a,Bartlett19benign,belkin2019two,hastie2019surprises,muthukumar2019harmless}.
These works bypass the classical approach to generalization outlined above and demonstrate that interpolating methods can generalize while tolerating some amount of noise. In particular, they show that interpolation can be ``harmless" in the sense that interpolating methods can in some cases achieve asymptotically optimal generalization error. At the same time, for the problems studied in these works there also exist non-interpolating algorithms with the same (or better) generalization guarantees. Thus these works do not explain why on many datasets (such as MNIST, CIFAR-10/100, SVHN) state-of-the-art classifiers interpolate the training data. We also remark that while interpolating the training set (with high generalization error) requires memorization, memorization also occurs without interpolation. For example, experiments of \citet{ZhangBHRV17} show that $9\%$ training error is achieved by a standard deep learning algorithm on completely randomly labeled 1000-class ImageNet dataset (with generalization error being $99.9\%$).

It is known that in the convex setting SGD converges faster when all the loss functions have a joint minimizer \citep{SrebroST10,NeedellWS14} and therefore it has been suggested that interpolation is the result of computational benefits of optimization via SGD \citep{MaBB18}. However this hypothesis is not well supported by empirical evidence since interpolation does not appear to significantly affect the speed with which the neural networks are trained \citep{ZhangBHRV17}. In addition, methods like nearest neighbors, boosting, and bagging are not trained via SGD but tend to interpolate the data as well.

Algorithmic stability \citep{BousquettE02,SSSSS:2009,HardtRS16,FeldmanV:19} is essentially the only general approach that is known to imply generalization bounds beyond those achievable via uniform convergence \citep{SSSSS:2009,Feldman:16erm}. However it runs into exactly the same conceptual issue as capacity-based bounds: average stability needs to be increased by at least $1/n$ to fit an arbitrary label. In fact, an interpolating learning algorithm does not satisfy any non-trivial uniform stability (but may still be on-average stable).

We focus on interpolation and the importance of label memorization in learning as this is the phenomenon that had no prior explanation. However neural networks are known to memorize much more than just labels \cite{carlini2019secret,CarliniTWJ20}. Such memorization presents even higher privacy risks and thus requires a more fundamental understanding. Building on the ideas in this work, recent work shows that for some natural data distributions, memorization of information about the entire sample can be necessary for achieving close-to-optimal generalization \cite{BrownBFST20}

%% file: interpolate-tails.tex
\section{Fitting the Training Data in Unstructured Classification}
\label{sec:main}
In this section we describe a simple learning setting over an unstructured discrete domain that incorporates a prior over the distribution of frequencies. We demonstrate that in the noiseless setting, a learning algorithm that does not fit the training examples will be suboptimal and express the excess error in terms of the properties of the prior over frequencies. We show that this result also holds in the presence of label noise (although only for the singleton examples). We then show that the excess error is significant if and only if the distribution of frequencies is long-tailed. Finally, we compare the conclusions of our analysis with those of the standard approaches in our setting.

\subsection{Preliminaries}
For a natural number $n$, we use $[n]$ to denote the set $\{1,\ldots,n\}$. For a condition $E$ (which defines a subset of some domain $X$) we use $\ind{E}$ to denote the indicator function of the condition (from $X$ to $\zo$). A dataset is specified by an ordered $n$-tuple of examples $S = ((x_1,y_1),\ldots,(x_n,y_n))$ but we will also treat it as the multi-set of examples it includes. Let $X_S$ denote the set of all points that appear in $S$.

For a probability distribution $D$ over $X$, $x\sim D$ denotes choosing $x$ by sampling it randomly from $D$. For subset (or condition) $E \subseteq X$ and function $F$ over $X$, we denote by $\Dist_{x\sim D}[F(x) \cond x  \in E]$ the probability distribution of $F(x)$, when $x \sim D$ and is conditioned on $x \in E$. For two probability distributions $D_1,D_2$ over the same domain we use $\TV(D_1,D_2)$ to denote the total variation distance between them.

The goal of the learning algorithm is to predict the labels given a dataset $S = ((x_1,y_1),\ldots,(x_n,y_n))$ consisting of i.i.d.~samples from some unknown distribution $P$ over $X \times Y$. For any function $h\colon X\to Y$ and distribution $P$ over $X\times Y$, we denote
$\err_{P}(h) \dfn \E_{(x,y)\sim P}[h(x) \neq y]$. As usual, for a randomized learning algorithm $\A$ we denote its expected generalization error on a dataset $S$ by
$$\err_P(\A,S) \dfn \E_{h\sim \A(S)} \lb \err_P(h) \rb ,$$
where $h\sim \A(S)$ refers to $h$ being the output of a (possibly) randomized algorithm. We also denote by $\err_P(\A) \dfn \E_{S\sim P^n} \lb \err_P(\A,S) \rb$
the expectation of the generalization error of $\A$ when examples are drawn randomly from $P$.

\subsection{Problem setup}
To capture the main phenomenon we are interested in, we start by considering a simple and general prediction problem in which the domain does not have any underlying structure (such as the notion of distance). The domains $X$ and $Y$ are discrete, $|X|=N$ and $|Y| = m$ (for concreteness one can think of $X=[N]$ and $Y=[m]$).

The prior information about the labels is encoded using a distribution $\F$ over functions from $X$ to $Y$.
The key assumption is that nothing is known a priori about the frequency of any individual point aside from a prior distribution over the individual frequencies. One natural approach to capturing this assumption is to assume that the frequencies of the elements in $X$ are known up to a permutation. That is, a distribution over $X$ is defined by picking a random permutation of elements of the prior $\prior = (\prior_1,\ldots,\prior_N)$. Exact knowledge of the entire frequency prior is also a rather strong assumption in most learning problems. We therefore use a related but different way to model the frequencies (which we have not encountered in prior work). In our model the frequency of each point in $X$ is chosen randomly and independently from the list of possible frequencies $\prior$ and then normalized to sum up to $1$.

More formally, let $\D_\prior^X$ denote the distribution over probability mass functions on $X$ defined as follows. For every $x\in X$, sample $p_x$ randomly, independently and uniformly from the elements of $\prior$. Define the corresponding probability mass function on $X$ as $D(x) = \frac{p_x}{\sum_{x\in X} p_x}$. This definition can be naturally generalized to sampling from a general distribution $\pi$ over frequencies (instead of just the uniform over a list of frequencies).
We also denote by $\piN$ the resulting marginal distribution over the frequency of any single element in $x$. That is, $$\piN(\alpha) \dfn \pr_{D\sim \D_\prior^X}[D(x)=\alpha] .$$  Note that, while $\prior$ is used to define the process, the actual distribution over individual frequencies the process results in is $\piN$ and our bounds will be stated in terms of properties of $\piN$. At the same time, this distinction is not particularly significant for applications of our result since, as we will show later, $\piN$ is essentially a slightly smoothed version of $\prior$.

The key property of this way to generate the frequency distribution is that it allows us to easily express the expected frequency of a sample conditioned on observing it in the dataset (or, equivalently, the mean of the posterior on the frequency). Specifically, in Appendix \ref{app:frequency} we prove the following lemma:
\begin{lem}
\label{lem:cond-density}
For any frequency prior $\pi$, $x \in X$ and a sequence of points $V = (x_1,\ldots,x_n)\in X^n$ that includes $x$ exactly $\ell$ times, we have
$$\E_{D \sim \D^X_\prior, U\sim D^n}[D(x) \cond U= V] =  \frac{\E_{\alpha \sim \piN} \lb \alpha^{\ell+1} \cdot (1-\alpha)^{n-\ell}\rb}{\E_{\alpha \sim \piN} \lb \alpha^\ell \cdot (1-\alpha)^{n-\ell}\rb }  .$$
\end{lem}

An instance of our learning problem is generated by picking a marginal distribution $D$ randomly from $\D_\prior^X$ and picking the true labeling function randomly according to $\F$. We refer to the distribution over $X\times Y$ obtain by picking $x\sim D$ and outputting   $(x,f(x))$ by $(D,f)$. We abbreviate $\D_\prior^X$ as $\D$ whenever the prior and $X$ are clear from the context.

We are interested in evaluating the generalization error of a classification algorithm on instances of our learning problem. Our results apply (via a simple adaption) to the more common setup in statistical learning theory where $F$ is a set of functions and worst case error with respect to a choice of $f\in F$ is considered. However for simplicity of notation and consistency with the random choice of $D$, we focus on the expectation of the generalization error on a randomly chosen learning problem:
$$\aerr(\prior,\F,\A) \dfn \E_{D\sim \D, f\sim \F} \lb \err_{D,f}(\A) \rb .$$

\subsection{The cost of not fitting}
We will now demonstrate that for our simple problem there exists a precise relationship between how well an algorithm fits the labels of the points it observed and the excess generalization error of the algorithm. This relationship will be determined by the prior $\piN$ and $n$. Importantly, this relationship will hold even when optimal achievable generalization error is high, a regime not covered by the usual analysis in the ``realizable" setting.

In our results the effect of not fitting an example depends on the number of times it occurs in the dataset and therefore we count examples that $\A$ does not fit separately for each possible multiplicity. More formally,
\begin{defn}
\label{def:forget}
For a dataset $S=((x_1,y_1),\ldots,(x_n,y_n))\in (X\times Y)^n$ and $\ell \in [n]$, let $X_{S\#\ell}$ denote the set of points $x$ that appear exactly $\ell$ times in $S$. For a function $h\colon X\to Y$ let
$$\errn_S(h,\ell) \dfn  \fr{\ell} \cdot |\{ i  \cond x_i \in X_{S\#\ell}\ \&\ h(x_i) \neq y_i \}| $$ and let
$$\errn_S(\A,\ell) \dfn \E_{h \sim \A(S)}[\errn_S(h,\ell)] .$$
\end{defn}

It is not hard to see (and we show this below) that in this noiseless setting the optimal expected generalization error is achieved by memorizing the dataset. Namely, by the algorithm that outputs the function that on the points in the dataset predicts the observed label and on points outside the dataset predicts the most likely label according to the posterior distribution on $\F$. We will now quantify the excess error of any algorithm that does not fit the labels of all the observed data points. Our result holds for every single dataset (and not just in expectation). To make this formal, we define $\cG$ to be the probability distribution over triples $(D,f,S)$ where $D\sim \D_\prior^X$, $f\sim \F$ and $S\sim (D,f)^n$. For any dataset $Z\in (X\times Y)^n$, let $\cG(|Z)$ denote the marginal distribution over distribution-function pairs conditioned on $S = Z$. That is:
$$\cG(|Z) \dfn \Dist_{(D,f,S)\sim \cG}[ (D,f) \cond S = Z] .$$
We then define the expected error of $\A$ conditioned on dataset being equal to $Z$ as
$$\aerr(\prior,\F,\A \cond Z) \dfn  \E_{(D,f) \sim \cG(|Z)} \lb \err_{D,f}(\A,Z) \rb .$$
We will also define $\opt(\prior,\F \cond Z)$ to be the minimum of $\aerr(\prior,\F,\A' \cond Z)$ over all algorithms $\A'$.
\begin{thm}
\label{thm:main-bound}
Let $\prior$ be a frequency prior with a corresponding marginal frequency distribution $\piN$, and $\F$ be a distribution over $Y^X$. Then for every learning algorithm $\A$ and every dataset $Z\in (X\times Y)^n$:
$$\aerr(\prior,\F,\A \cond Z) \geq \opt(\prior,\F \cond Z) + \sum_{\ell \in [n]} \tau_\ell \cdot \errn_Z(\A,\ell), $$ where
$$\tau_\ell \dfn  \frac{\E_{\alpha \sim \piN} \lb \alpha^{\ell+1} \cdot (1-\alpha)^{n-\ell}\rb}{\E_{\alpha \sim \piN} \lb \alpha^\ell \cdot (1-\alpha)^{n-\ell}\rb } .$$
In particular,
$$\aerr(\prior,\F,\A) \geq \opt(\prior,\F) +  \E_{D\sim \D_\pi^X, f\sim \F, S\sim (D,f)^n} \lb \sum_{\ell \in [n]}  \tau_\ell \cdot \errn_S(\A,\ell) \rb .$$
\end{thm}
\begin{proof}
We denote the marginal distribution of $\cG(|Z)$ over $D$ by $\D(|Z)$ and the marginal distribution over $f$ by $\F(|Z)$.
We begin by noting that for every $f'\colon X\to Y$ consistent with the examples in $Z$, the distribution of $D$ conditioned on $f=f'$ is still $\D(|Z)$, since $D$ is chosen independently of any labeling. Therefore we can conclude that $\cG(|Z)$ is equal to the product distribution $\D(|Z) \times \F(|Z)$.

To prove the claim we will prove that
\equ{\aerr(\prior,\F,\A \cond Z) = \sum_{\ell \in [n]} \tau_\ell \cdot \errn_Z(\A,\ell) + \sum_{x\in X_{Z\#0}} \pr_{h \sim \A(Z), f\sim \F(|Z)}[h(x) \neq f(x)]  \cdot p(x,Z), \label{eq:main-per-s}}
where $p(x,Z) \dfn \E_{D \sim \D(|Z)}[D(x)]$.
This will imply the claim since the right-hand expression is minimized when for all $\ell \in [n]$, $\errn_Z(\A,\ell) = 0$ and for all $x\in X_{Z\#0}$, $$ \pr_{h \sim \A(Z), f\sim \F(|Z)}[h(x) \neq f(x)]  = \min_{y \in Y} \pr_{f\sim \F(|Z)}[f(x) \neq y].$$ Moreover, this minimum is achieved by the algorithm $\A^\ast$ that fits the examples in $Z$ and predicts the label $y$ that minimizes $\pr_{f\sim \F(|Z)}[f(x) \neq y]$ on all the points in $X_{Z\#0}$. Namely,
\alequn{ \sum_{x\in X_{Z\#0}} \pr_{h \sim \A(Z), f\sim \F(|Z)}[h(x) \neq f(x)]  \cdot p(x,Z) & \geq \sum_{x\in X_{Z\#0}} \min_{y \in Y} \pr_{f\sim \F(|Z)}[f(x) \neq y]  \cdot p(x,Z)\\
 & = \aerr(\prior,\F,\A^\ast \cond Z) = \opt(\prior,\F \cond Z) .}
Plugging this into eq.~\eqref{eq:main-per-s} gives the first claim.

We now prove eq.~\eqref{eq:main-per-s}.
\alequ{\aerr(\prior,\F,\A \cond Z) &= \E_{(D,f) \sim \cG(|Z), h\sim \A(Z)} \lb \err_{D,f}(h) \rb \nonumber\\
& = \E_{(D,f) \sim \cG(|Z), h\sim \A(Z)} \lb \sum_{x\in X} \ind{h(x) \neq f(x)} \cdot D(x) \rb \nonumber \\
& = \sum_{x\in X_Z} \E_{(D,f) \sim \cG(|Z), h\sim \A(Z)} \lb  \ind{h(x) \neq f(x)} \cdot D(x) \rb \label{eq:in-sample} \\
& + \sum_{x\in X_{Z\#0}} \E_{(D,f) \sim \cG(|Z), h\sim \A(Z)} \lb  \ind{h(x) \neq f(x)} \cdot D(x)  \rb \label{eq:out-sample}.
}

Using the fact that $\cG(|Z) = \D(|Z) \times \F(|Z)$, for every $x \in X_{Z\#0}$ we get
\alequn{\E_{(D,f) \sim \cG(|Z),  h\sim \A(Z)} \lb  \ind{h(x) \neq f(x)} \cdot D(x)  \rb &= \pr_{h\sim \A(Z), f\sim \F(|Z)}[h(x) \neq f(x)] \cdot \E_{D\sim \D(|Z)} \lb D(x) \rb\\ & =
\pr_{h\sim \A(Z), f\sim \F(|Z)}[h(x) \neq f(x)] \cdot p(x,Z).
}
Hence we obtain that the term in line \eqref{eq:out-sample} is exactly equal to the second term on the right hand side of eq.~\eqref{eq:main-per-s}.

For the term in line \eqref{eq:in-sample}, we pick an arbitrary $x\in X_{Z\#\ell}$ for some $\ell \in [n]$. 
We can decompose
$$\E_{(D,f) \sim \cG(|Z), h\sim \A(Z)} \lb  \ind{h(x) \neq f(x)} \cdot D(x) \rb = \pr_{h\sim \A(Z), f\sim \F(|Z)}[h(x) \neq f(x)] \cdot \E_{D \sim \D(|Z)} [D(x)] $$
since additional conditioning on $h(x) \neq f(x)$ does not affect the distribution of $D(x)$ (as mentioned, $\cG(|Z)$ is a product distribution). Let $V$ denote the sequence of points in the dataset $Z=((x_1,y_1),\ldots,(x_n,y_n))$. The labels of these points do not affect the conditioning of $D$ and therefore by Lemma \ref{lem:cond-density},
$$  \E_{D \sim \D(|Z)} [D(x)] = \E_{D \sim \D, U\sim D^n}[D(x) \cond U= V] = \frac{\E_{\alpha \sim \piN} \lb \alpha^{\ell+1} \cdot (1-\alpha)^{n-\ell}\rb}{\E_{\alpha \sim \piN} \lb \alpha^\ell \cdot (1-\alpha)^{n-\ell}\rb } = \tau_\ell . $$
For a point $x \in X_Z$, we denote by $Z(x)$ the label of $x$ in $Z$. This label is unique in our setting and is equal to $f(x)$ for every $f$ in the support of $\F(|Z)$. Therefore, by combining the above two equalities we obtain that, as claimed in eq.\eqref{eq:main-per-s}, line \eqref{eq:in-sample} is equal to
\alequn{\eqref{eq:in-sample} & = \sum_{i\in [n], x\in X_Z} \E_{(D,f) \sim \cG(|Z), h\sim \A(Z)} \lb  \ind{h(x) \neq f(x)} \cdot D(x) \rb \\
& = \sum_{\ell \in [n]} \sum_{x \in X_{Z\#\ell}} \tau_\ell \cdot \pr_{h\sim \A(Z)}[h(x) \neq Z(x)] \\
& = \sum_{\ell \in [n]} \tau_\ell \cdot \errn_Z(\A,\ell) .
}

To obtain the second part of the theorem we denote by $\cS$ the marginal distribution of $\cG$ over $S$. Observe that $$\opt(\prior,\F) = \E_{Z \sim \cS} \lb \opt(\prior,\F \cond Z) \rb $$ since the optimal algorithm is given $Z$ as an input. The second claim now follows by taking the expectation over the marginal distribution over $S$:
\alequn{\aerr(\prior,\F,\A)  &= \E_{Z \sim \cS} [\aerr(\prior,\F,\A \cond Z)]\\
& \geq \E_{Z \sim \cS} \lb \opt(\prior,\F \cond Z)  + \sum_{\ell \in [n]} \tau_\ell \cdot \errn_Z(\A,\ell)\rb \\
& = \opt(\prior,\F) +  \sum_{\ell \in [n]} \tau_\ell \cdot \E_{Z \sim \cS} \lb \errn_Z(\A,\ell)\rb .
}
\end{proof}

\subsection{Extension to label noise}
\label{sec:noise}
A more general way to view Theorem \ref{thm:main-bound} is that it translates excess error on points in $X_{S\#\ell}$ into excess generalization error (excess error on points in $X_{S\#\ell}$ is the difference between the total error of $\A$ on $X_{S\#\ell}$ and the error of the optimal algorithm on $X_{S\#\ell}$). This view holds even if we allow noise in the labels. In the presence of noise the observed labels are not necessarily correct and therefore the error of $\A$ on $X_{S\#\ell}$ may no longer be equal to the empirical error $\errn_S(\A,\ell)$. At the same time, if for a singleton example $(x,y)$, the posterior probability of label $y$ on $x$ is higher than that of other labels, then fitting label $y$ on $x$ is still the optimal strategy. In this case any algorithm that does not do that will be suboptimal by at least by $\tau_1$ (for every such example). When the noise level is relatively low and affects primarily hard examples (which is the case in most standard benchmark datasets), the observed label is much more likely to be the correct one than the other labels. Thus on such datasets it is optimal to fit even noisy labels.

To make this argument formal we consider a more general setting in which for every true labeling function $f$ the examples are labeled by some $\tilde f$. Formally, we assume that there is a possibly randomized mapping from the support of $\F$ to $Y^X$ and sampling of $f$ from $\F$ also includes $\tilde f$. In particular, in the conditional  probability $\F\cond Z$ we include the randomness with respect to generation of $\tilde f$ (that labeled $Z$) from $f$. Further, it is natural to assume that for a singleton example its label given by $\hat f$ is the most likely to be correct by some margin even conditioned on the rest of the dataset. Formally, we denote the confidence margin in the given label for the given prior $\F$ as
\equ{ \conf(Z,i,\F) \dfn \min\left\{0, \pr_{f \sim \F\cond Z}[f(x_i) = y_i] - \max_{y \in Y\setminus\{y_i\}} \pr_{f\sim \F\cond Z}[f(x) = y]\right\}. \label{eq:def-conf}}
\begin{thm}
\label{thm:main-noise}
Using the notation in (the proof of) Theorem \ref{thm:main-bound}, we have that
$$\aerr(\prior,\F,\A \cond Z) \geq \opt(\prior,\F \cond Z) + \tau_1 \cdot \sum_{i\in [n], x_i\in X_{Z\#1}}  \conf(Z,i,\F)\cdot \pr_{h\sim \A(Z)}[h(x_i) \neq y_i]  $$
In particular,
$$\aerr(\prior,\F,\A) \geq \opt(\prior,\F) + \tau_1 \cdot \E_{D\sim \D_\pi^X, f\sim \F, S\sim (D,\tilde f)^n}  \lb \sum_{i\in [n], x_i\in X_{S\#1}} \conf(S,i,\F)\cdot \pr_{h\sim \A(S)}[h(x_i) \neq y_i]  \rb .$$
\end{thm}
\begin{proof}
As in the proof of Theorem \ref{thm:main-bound}, the fact that $\cG(|Z) = \D(|Z) \times \F(|Z)$ implies that
\equ{\aerr(\prior,\F,\A \cond Z) = \sum_{x\in X} \pr_{h \sim \A(Z), f\sim \F(|Z)}[h(x) \neq f(x)]  \cdot p(x,Z), \label{eq:main-per-s-gen}}
where, as before, $p(x,Z) \dfn \E_{D \sim \D(|Z)}[D(x)]$.
This implies that
\alequ{\aerr(\prior,\F,\A \cond Z) &- \opt(\prior,\F \cond Z) \label{eq:decompose-subopt} \\
& = \sum_{x\in X} \left( \pr_{h \sim \A(Z), f\sim \F(|Z)}[h(x) \neq f(x)] - \min_{y\in Y} \pr_{f\sim \F(|Z)}[f(x) \neq y]\right) \cdot p(x,Z) \nonumber \\
&\geq \sum_{x\in X_{S\#1}} \left( \pr_{h \sim \A(Z), f\sim \F(|Z)}[h(x) \neq f(x)] - \min_{y\in Y} \pr_{f\sim \F(|Z)}[f(x) \neq y]\right) \cdot \tau_1 .\nonumber
}

By our definition in eq.~\eqref{eq:def-conf}, for every $x_i\in X_{Z\#1}$, if $h(x_i) \neq y_i$ then
\alequn{ \pr_{f\sim \F(|Z)}[h(x_i) \neq f(x_i)] - \min_{y\in Y} \pr_{f\sim \F(|Z)}[f(x_i) \neq y] & 
= \max_{y\in Y} \pr_{f\sim \F(|Z)}[f(x_i) = y] - \pr_{f\sim \F(|Z)}[h(x_i) = f(x_i)] \\
&\geq \conf(Z,i,\F) .}
Substituting this into eq.~\eqref{eq:decompose-subopt}, we obtain the claimed result.
\end{proof}

\subsection{From tails to bounds}
\label{sec:tail-bounds}
Given a frequency prior $\pi$, Theorem \ref{thm:main-bound} gives a general and easy way to compute the effect of not fitting an example in the dataset. We now spell out some simple and easier to interpret corollaries of this general result and show that the effect can be very significant. The primary case of interest is $\ell =1$, namely examples that appear only once in $S$, which we refer to as {\em singleton} examples. In order to fit those, an algorithm needs to memorize their labels whenever $\F$ is hard to learn (see Section~\ref{sec:memory} for a more detailed discussion). We first note that the expected number of singleton examples is determined by the weight of the entire tail of frequencies below $1/n$ in $\piN$. Specifically, the expected fraction of the distribution $D$ contributed by frequencies in the range $[\beta_1,\beta_2]$ is defined as:
\alequn{\weight(\piN, [\alpha,\beta]) & \dfn \E_{D\sim \D}\lb \sum_{x\in X} D(x) \cdot \ind{D(x)\in [\beta_1,\beta_2]} \rb \\
&= N \cdot \E_{\alpha \sim \piN}\lb  \alpha \cdot \ind{\alpha \in [\beta_1,\beta_2]} \rb .}
At the same time the expected number of singleton points is:
\alequn{\single(\piN)& \dfn \E_{D\sim \D, V \sim D^n} \lb |X_{V=1}|\rb = \E_{D\sim \D}\lb \sum_{x\in X} \pr_{V\sim D^n} [x\in X_{V=1}] \rb \\
&= \E_{D\sim \D}\lb \sum_{x\in X} n \cdot D(x) (1-D(x))^{n-1} \rb \\
&= \sum_{x\in X} n \E_{D\sim \D}\lb D(x) (1-D(x))^{n-1} \rb \\
&= n  N \cdot \E_{\alpha \sim \piN} \lb \alpha (1-\alpha)^{n-1}\rb .
}

For every $\alpha \leq 1/n$ we have that $(1-\alpha)^{n-1} \geq 1/3$ (for sufficiently large $n$). Therefore:
\equ{\single(\piN) \geq n  N \cdot \E_{\alpha \sim \piN} \lb \alpha (1-\alpha)^{n-1} \cdot \ind{\alpha \in \lb 0,\fr{n}\rb} \rb \geq \frac{n}{3} \cdot \weight\lp\piN, \lb 0,\fr{n}\rb \rp . \label{eq:singleton-weight}}

We will now show that the expected cost of not fitting any of the singleton examples is lower bounded by the weight contributed by frequencies on the order of $1/n$. Our bounds will be stated in terms of the properties of $\piN$ (as opposed to $\prior$ itself) and therefore, before proceeding, we briefly explain the relationship between these two.

\paragraph{Relationship between $\prior$ and $\piN$:}
Before the normalization step, for every $x \in X$, $p_x$ is distributed exactly according to $\prior$ (that is uniform over $(\prior_1,\ldots,\prior_N)$. Therefore, it is sufficient to understand the distribution of the normalization factor conditioned on $p_x = \prior_i$ for some $i$. Under this condition the normalization factor $s_i$ is distributed as the sum of $n-1$ independent samples from $\pi$ plus $\prior_i$. The mean of each sample is exactly $1/N$ and thus standard concentration results can be used to obtain that $s_i$ is concentrated around $\frac{N-1}{N} + \prior_i$. Tightness of this concentration depends on the properties of $\prior$, most importantly, the largest value $\prior_{\max} \dfn \max_{j\in [N]} \prior_j$ and $\Var[\prior] \dfn \fr{N}\sum_{j\in [N]}(\prior_j-\fr{N})^2 \leq \prior_{\max}$. For $\prior_{\max} = o(1)$, $\piN$ can be effectively seen as convolving each $\prior_i$ multiplicatively by a factor whose inverse is a Gaussian-like distribution of mean  $1-1/N + \prior_i$ and variance $\Var(\prior)$. More formally, using Bernstein's (or Bennett's) concentration inequality (\eg \citep{sridharan2002gentle}) we can easily relate the total weight in a certain range of frequencies under $\piN$ to the weight in a similar range under $\prior$.
\begin{lem}
\label{lem:relate-priors}
Let $\prior = (\prior_1,\ldots,\prior_N)$ be a frequency prior and $\piN$ be the corresponding marginal distribution over frequencies. For any $0<\beta_1<\beta_2<1$ 
Then for
and any $\gamma > 0$,
$$ \weight(\piN,[\beta_1,\beta_2])  \geq \frac{(1-\delta)}{1-\fr{N} + \beta_2 + \gamma} \cdot \weight\lp\prior, \lb \frac{\beta_1}{1-\fr{N} + \beta_1 - \gamma}, \frac{\beta_2}{1-\fr{N} + \beta_2 + \gamma} \rb\rp ,$$
where
$\prior_{\max} \dfn \max_{j\in [N]} \prior_j$, $\Var[\prior] \dfn \sum_{j\in [N]}(\prior_j-\fr{N})^2$ and $\delta \dfn 2 \cdot e^{\frac{-\gamma^2}{2 (N-1) \Var(\prior) + 2 \gamma \prior_{\max}/3}}$.
\end{lem}
Note that $$\Var[\pi] \leq  \fr{N}\sum_{j\in [N]}\prior_j^2 \leq\frac{\prior_{\max}}{N} \cdot \sum_{j\in [N]}\prior_j = \frac{\prior_{\max}}{N}.$$ By taking $\gamma =1/4$, we can ensure that the boundaries of the frequency interval change by a factor of at most (roughly) $4/3$. For such $\gamma$ we will obtain $\delta \leq 2e^{-1/(40\prior_{\max})}$ and in particular $\prior_{\max} \leq 1/200$ will suffice for making the correction $(1-\delta)$ at least $99/100$ (which is insignificant for our purposes).

\paragraph{Bounds for $\ell =1$:}
We now show a simple lower bound on $\tau_1$ in terms of $\weight(\piN, [1/2n,1/n])$ (similar results hold for other choices of the interval $[c_1/n,c_2/n]$). We also do not optimize the constants in the bounds as our goal is to demonstrate the qualitative behavior.
\begin{lem}
\label{lem:bound-tai1}
For every frequency prior $\prior$ and sufficiently large $n,N$,
$$\tau_1 \geq \fr{5n} \cdot \weight\lp \piN, \lb\fr{3n},\frac{2}{n}\rb\rp  .$$
If, in addition,  $\prior_{\max} \leq 1/200$, then
$$\tau_1 \geq \fr{7n} \cdot \weight\lp \prior, \lb\fr{2n},\frac{1}{n}\rb\rp .$$
\end{lem}
\begin{proof}
We first observe that the denominator of $\tau_1$ satisfies
$$ \E_{\alpha \sim \piN} \lb \alpha (1-\alpha)^{n-1}\rb \leq \E_{\alpha \sim \piN} \lb \alpha \rb = \fr{N}.$$
Now, by simple calculus, for every $\alpha \in \lb\fr{3n},\frac{2}{n}\rb $ and sufficiently large $n$,
$$\alpha^2  (1-\alpha)^{n-1} \geq \fr{5n} \cdot \alpha .$$
Therefore
\alequn{\tau_1 &=  \frac{\E_{\alpha \sim \piN} \lb \alpha^2  (1-\alpha)^{n-1}\rb}{\E_{\alpha \sim \piN} \lb \alpha (1-\alpha)^{n-1}\rb} \\
& \geq \frac{\fr{5n} \cdot \E_{\alpha \sim \piN}\lb \alpha \cdot \ind{\alpha \in \lb \fr{3n},\frac{2}{n}\rb}\rb }{\fr{N}} = \fr{5n} \cdot \weight\lp \piN, \lb\fr{3n},\frac{2}{n}\rb\rp .}

To obtain the second part of the claim we apply Lemma \ref{lem:relate-priors} for $\gamma = 1/4$ (as discussed above). To verify, observe that for sufficiently large $n$ and $N$, $\frac{\fr{3n}}{1-\fr{N} + \fr{3n} - 1/4} \leq \fr{2n}$ and $\frac{\frac{2}{n}}{1-\fr{N} + \frac{2}{n} + 1/4} \geq \fr{n}$, and
$\frac{(1-\delta)}{1-\fr{N} + \frac{2}{n} + \gamma} \geq \frac{3}{4}$.
\end{proof}
The value of $\tau_1 = \Omega(1/n)$ corresponds to paying on the order of $1/n$ in generalization error for every example that is not fit by the algorithm. Hence if the total weight of frequencies in the range of $1/n$ is at least some $\theta$ then the algorithm that does not fit them will be suboptimal by $\theta$ times the fraction of such examples in the dataset. By eq.~\eqref{eq:singleton-weight}, the fraction of such examples themselves is determined by the weight of the entire tail $\weight(\piN,[0,1/n])$. For example, if $\prior$ is the Zipf distribution and $N \geq n$ then $\tau_1 = \Omega(1/n)$ and $\weight(\piN,[0,1/n]) = \Omega(1)$. Thus an algorithm that does not fit most of the singleton examples will be suboptimal by $\Omega(1)$. Numerically, for $N=n=50,000$ an algorithm that in a binary prediction problem does no better than random on the singletons will have excess error of $4\%$ (relative to the optimum which is $8.5\%$ in this case).

We can contrast this situation with the case where there are no frequencies that are on the order of $1/n$. Even when the data distribution has no elements with such frequency, the total weight of the frequencies in the tail and as a result the fraction of singleton points might be large. Still, as we show, in such case the cost of not fitting singleton examples will be negligible.
\begin{lem}
\label{lem:no-middle}
Let $\prior$ be a frequency prior such that for some $\theta \leq \fr{2n}$,
$\weight\lp \piN, \lb \theta, \frac{t}{n} \rb\rp = 0$, where $t =\ln(1/(\theta\beta))+2$ for $\beta \dfn \weight\lp \piN, [0, \theta]\rp$.
Then $\tau_1 \leq 2\theta$.
\end{lem}
\begin{proof}
We first observe that the numerator of $\tau_1$ is at most:
\alequn{\E_{\alpha \sim \piN} \lb \alpha^2 (1-\alpha)^{n-1}\rb & \leq \max_{\alpha \in [t/n,1]} \alpha^2 (1-\alpha)^{n-1} \cdot \pr_{\alpha \sim \piN}\lb \alpha \geq \frac{t}{n} \rb\\
& + \E_{\alpha \sim \piN} \lb \alpha^2 (1-\alpha)^{n-1} \cdot \ind{\alpha \leq \theta}\rb .
}
By Markov's inequality, $\E_{\alpha \sim \piN} [\alpha] = \fr{N}$ implies $$\pr_{\alpha \sim \piN}\lb \alpha \geq \frac{t}{n} \rb \leq \frac{n}{t N}.$$
In addition, by our definition of $t$,
$$ \max_{\alpha \in [t/n,1]} \alpha^2 (1-\alpha)^{n-1} \leq \frac{t}{n} \lp 1-\frac{t}{n}\rp^{n-1} \leq \frac{t \beta \theta}{en} .$$
Therefore the first term in the numerator is upper bounded by $\frac{n}{t N} \frac{t \beta \theta}{en} \leq \frac{\beta \theta}{eN}$.
At the same time the second term in the numerator satisfies:
\alequn{ \E_{\alpha \sim \piN} \lb \alpha^2 (1-\alpha)^{n-1} \cdot \ind{\alpha \leq \theta}\rb & \geq \theta (1-\theta)^{n-1} \cdot \E_{\alpha \sim \piN} \lb \alpha \cdot \ind{\alpha \leq \theta}\rb \\
& \geq \theta \lp 1-\fr{2n}\rp^{n-1} \cdot \frac{\weight\lp \piN, [0, \theta]\rp}{N} \geq \frac{\theta\beta}{2N}.
}
Therefore the second term is at least as large as the first term and we obtain that:
\alequn{\E_{\alpha \sim \piN} \lb \alpha^2 (1-\alpha)^{n-1}\rb & \leq 2 \cdot  \E_{\alpha \sim \piN} \lb \alpha^2 (1-\alpha)^{n-1} \cdot \ind{\alpha \leq \theta}\rb \\
&\leq 2\theta \cdot  \E_{\alpha \sim \piN} \lb \alpha (1-\alpha)^{n-1} \cdot \ind{\alpha \leq \theta}\rb \\
&\leq 2\theta \cdot  \E_{\alpha \sim \piN} \lb \alpha (1-\alpha)^{n-1}\rb .}
Thus $\tau_1 \leq 2\theta$ as claimed.
\end{proof}
For $\theta = 1/(2n^2)$, under the conditions of Lemma \ref{lem:no-middle} we will obtain that the suboptimality of the algorithm that does not fit any of the singleton examples is at most $1/n$.

\subsection{Comparison with standard approaches to generalization}
\label{sec:compare-standard}
We now briefly demonstrate that standard approaches for analysis of generalization error cannot be used to derive the conclusions of this section and do not capture our simple problem whenever $N \geq n$. For concreteness, we will use $m=2$ with the uniform prior over all labelings. We will also think of $\pi$ that consists of $n/2$ frequencies $1/n$ and $n^2/2$ frequencies $1/n^2$ (thus $N=n^2/2+n/2$). Without any structure in the labels, a natural class of algorithms for the problem are algorithms that pick a subset of points whose labels are memorized and predict randomly on the other points in the domain.

First of all, it is clear that any approach that does not make any assumption on the marginal distribution $D$ cannot adequately capture the generalization error of such algorithms. A distribution-independent generalization bound needs to apply to the uniform distribution over $X$. For this distribution the expected generalization error for a randomly chosen labeling function $f$ will be at least $(1 - n/N)/2\approx 0.5$. In particular, for sufficiently large $N$, the differences in the generalization error of different algorithms will be insignificant and therefore such notion will not be useful for guiding the choice of the algorithm.

Notions that are based on the algorithm knowing the input distribution $D$ are not applicable to our setting. Indeed the main difficulty is that the algorithm does not know the exact frequencies of the singleton elements. An algorithm that knows $D$ would not need to fit the points whose frequency is $1/n^2$. Thus the algorithm would be able to achieve excess generalization error of at most $1/n$ without fitting the dataset.  In contrast, our analysis shows that an algorithm that only knows the prior and fits only $50\%$ of the dataset will be suboptimal by $>13\%$.

Fairly tight data-dependent bounds on the generalization error can be obtained via the notion of empirical Rademacher complexity \citep{koltchinskii2001rademacher,BartlettMendelson:02}. Empirical Rademacher complexity for a dataset $S$ and the class of all Boolean functions on $X$ that memorize $k$ points is $\geq \min\{k,|X_S|\}/n$. Similar bound can also be obtained via weak notions of stability such as average leave-one-out stability \citep{BousquettE02,RakhlinMP05,MukherjeeNPR06,Shalev-ShwartzSSS10}
\equ{\mathtt{LOOstab}(P,\A) \dfn \fr{n} \sum_{i\in [n]} \E_{S \sim P^n}\lb \left| \pr_{h\sim \A(S)}[h(x_i)=y_i] - \pr_{h\sim \A(S^{\setminus i})}[h(x_i)=y_i]  \right|\rb ,\label{eq:loo-stab}}
where $S^{\setminus i}$ refers to $S$ with $i$-th example removed. 
If we were to use either of these notions to pick $k$ (the number of points to memorize), we would end up not fitting any of the singleton points. The simple reason for this is that, just like a learning algorithm cannot distinguish between ``outlier" and ``atypical" points given $S$ in this setting, neither will any bound. Therefore any true upper bound on the generalization error that is not aware of the prior on the frequencies needs to be correct when all the points that occur once are ``outliers''. Fitting any of the outliers does not improve the generalization error at all and therefore such upper bounds on the generalization error cannot be used to correctly guide the choice of $k$.

An additional issue with the standard approaches to analysis of the generalization error is that they bound the excess error of an algorithm relative to the best function in some class of functions or relative to the Bayes optimal predictor (which is the the optimal predictor for the true data distribution). In our model this would mean comparing with the perfect predictor which has generalization error of $0$. For the prior $\pi$ we consider, the optimal algorithm has generalization error of over $25\%$. Thus theoretical analysis that is not close-to-perfectly tight will not lead to a meaningful bound. For example, standard bounds based on Rademacher complexity are suboptimal by a factor of at least two and thus lead to vacuous bounds. In contrast, our analysis can give a meaningful bound on the generalization error even when used with a relatively crude bound on the excess error.

\section{General Mixture Models}
\label{sec:mixtures}
Our problem setting in Section \ref{sec:main} considers discrete domains without any structure on $X$. The results also focus on elements of the domain whose frequency is on the order of $1/n$. Naturally, practical prediction problems are often high-dimensional with each individual point having an exponentially small (in the dimension) probability. Therefore direct application of our analysis from Section \ref{sec:main} for the unstructured case makes little sense. Indeed, any learning algorithm $\A$ can be modified to a learning algorithm $\A'$ that does not fit any of the points in the dataset and achieves basically the same generalization error as $\A$ simply by modifying $\A$'s predictions on the training data to different labels and vice versa (any algorithm can be made to fit the dataset without any effect on its generalization).

At the same time in high dimensional settings the points have additional structure that can be exploited by a learning algorithm. Most machine learning algorithms are very likely to produce the same prediction on points that are sufficiently ``close" in some representation. The representation itself may be designed based on domain knowledge or derived from data. This is clearly true about $k$-NN, SVMs/linear predictors and has been empirically observed for neural networks once the trained representation in the last hidden layer is considered.

The second important aspect of natural image and text data is that it can be viewed as a mixture of numerous subpopulations.
As we have discussed in the introduction, the relative frequency of these subpopulations has been observed to have a long-tailed distribution most obvious when considering the label distribution in extreme multiclass problems \citep{zhu2014capturing,babbar2017dismec,wang2017learning,krishna2017visual,van2017devil,Cui_2018_CVPR,van2018inaturalist,Babbar2019} (see also Fig.~\ref{fig:long}). A natural way to think of and a common way to model subpopulations (or mixture components) is as consisting of points that are similar to each other yet sufficiently different from other points in the domain.

We capture the essence of these two properties using the following model that applies the ideas we developed in Section \ref{sec:main} to mixture models. To keep the main points clear we keep the model relatively simple by making relatively strong assumptions on the structure. (We discuss several ways in which the model's assumptions can be relaxed or generalized later).

We model the unlabeled data distribution as a mixture of a large number of fixed distributions $M_1,\ldots,M_N$. For simplicity, we assume that these distributions have disjoint support, namely $M_i$ is supported over $X_i$ and $X_i \cap X_j = \emptyset$ for $i\neq j$ (without loss of generality $X = \cup_{i\in [N]} X_i$). For $x\in X$ we denote $i_x$ to be the index of the sub-domain of $x$ and by $X_x$ (or $M_x$) the sub-domain (or subpopulation, respectively) itself.

The unknown marginal distribution $M$ is defined as $M(x) \dfn \sum_{i\in [N]} \alpha_i M_i(x)$ for some vector of mixture coefficients $(\alpha_1,\ldots,\alpha_N)$ that sums up to $1$. We describe it as a distribution $D(x)$ over $[N]$ (that is $\alpha_i = D(i)$). As in our unstructured model, we assume that nothing is known a priori about the mixture coefficients aside from (possibly) a prior $\prior = (\prior_1,\ldots,\prior_N)$ described by a list of frequencies. The mixture coefficients are generated, as before, by sampling $D$ from $\D_\prior^{[N]}$. We denote by $M_D$ the distribution over $X$ defined as $M_D(x) \dfn \sum_{i\in [N]} D(i) M_i(x)$.

We assume that the entire subpopulation $X_i$ is labeled by the same label and the label prior is captured via an arbitrary distribution $\F$ over functions from $[N]$ to $Y$. Note that such prior can be used to reflect a common situation where a subpopulation that is ``close"  to subpopulations $i_1$ and $i_2$ is likely to have the same label as either $i_1$ or $i_2$. The labeling function $L$ for the entire domain $X$ is sampled by first sampling $f\sim \F$ and defining $L_f(x) = f(i_x)$.

To model the properties of the learning algorithm we assume that for every point $x$ in a dataset $S$ the distribution over predictions $h(x)$ for a random predictor output by $\A(S)$ is close to (or at least not too different) from the distribution over predictions that $\A$ produces over the entire subpopulation of $x$. This follows the intuition that labeling $x$ will have a measurable effect on the prediction over the entire subpopulation. This effect may depend on the number of other points from the same subpopulation and therefore our assumption will be parameterized by $n$ parameters.
\begin{defn}
\label{def:subpop-coupled}
Let $X$ be a domain partitioned into sub-domains $\{X_i\}_{i\in [N]}$ with subpopulations $\{M_i\}_{i\in [N]}$ over the sub-domains.
For a dataset $S$, let $X_{S\#\ell}$ denote the union of subpopulations $X_i$ such that points from $X_i$ appear exactly $\ell$ times in $S$. For $\Lambda=(\lambda_1,\ldots,\lambda_n)$, we say that an algorithm $\A$ is {\em $\Lambda$-subpopulation-coupled} if for every $S \in (X\times Y)^n$, $x\in X_{S\#\ell}$,
$$\TV\lp \Dist_{h\sim \A(S)} [h(x)], \Dist_{x' \sim M_x, h\sim \A(S)} [h(x')] \rp \leq 1-\lambda_\ell .$$
\end{defn}

Note that we do not restrict the algorithm to be coupled in this sense over subpopulations that are not represented in the data. This distinction is important since predictors output by most natural algorithms vary over regions from which no examples were observed. 
As a result the setting here cannot be derived by simply collapsing points in the sub-domain into a single point and applying the results from the unstructured case. However, the analysis and the results in Sec.~\ref{sec:main} still apply essentially verbatim to this more general setup. All we need is to extend the definition of $\errn_S(\A,\ell)$ to look at the multiplicity of sub-domains and not points themselves and count mistakes just once per sub-domain.
For a function $h\colon X\to Y$ let $$\errn_S(h,\ell) = \fr{\ell} \sum_{i\in [n]} \ind{x_i \in X_{S\#\ell}\mbox { and } h(x_i) \neq y_i} .$$ As before, $\errn_S(\A,\ell) = \E_{h\sim \A(S)}[\errn_S(h,\ell)]$. With this definition we get the following generalization of Theorem \ref{sec:main} (we only state the version for the total expectation of the error but the per-dataset version holds as well):
\begin{thm}
\label{thm:main-populations}
Let $\{M_i\}_{i\in [N]}$ be subpopulations over sub-domains $\{X_i\}_{i\in [N]}$ and let $\prior$ and $\F$ be some frequency and label priors.
Then for every $\Lambda$-subpopulation-coupled learning algorithm $\A$:
$$\aerr(\prior,\F,\A) \geq \opt(\prior,\F) +  \E_{D\sim \D_\pi^{[N]}, f\sim \F, S\sim (M_D,L_f)^n} \lb \sum_{\ell \in [n]} \lambda_\ell \tau_\ell \cdot \errn_S(\A,\ell) \rb, $$
where $\tau_\ell$ is defined in Thm.~\ref{thm:main-bound}.
\end{thm}

We now briefly discuss how the modeling assumptions can be relaxed. We first note that it suffices for subpopulation coupling to hold with high probability over the choice of dataset $S$ from the marginal distribution over the datasets $\cS$. Namely, if the property in Definition \ref{def:subpop-coupled} holds with probability $1-\delta$ over the choice of $S \sim \cS$ (where, $\cS$ is the marginal distribution over the datasets) then the conclusion of the theorem holds up to an additional $\delta$. This follows immediately from the fact that Theorem \ref{thm:main-populations} holds for every dataset separately.

The assumption that the components of the mixture are supported on disjoint subdomains is potentially quite restrictive as it does not allow for ambiguous data points (for which Bayes optimal error is $>0$). Subpopulations are also often modeled as Gaussians (or other distributions with unbounded support). If the probability of the overlap between the subpopulations is sufficiently small, then one can reduce this case to the disjoint one by modifying the components $M_i$ to have disjoint supports while changing the marginal distribution over $S$ by at most $\delta$ in the TV distance (and then appealing to the same argument as above). Dealing with a more general case allowing general overlap is significantly messier but the basic insight still applies: observing a single point sampled from some subpopulation increases the expectation of the frequency of the subpopulation under the posterior distribution. That increase can make this expectation significant making it necessary to memorize the label of the point.



\subsection{Examples}
\label{sec:examples}
We will now provide some intuition on why one would expect the $\Lambda$-subpopulation-coupling to hold for some natural classes of algorithms. Our goal here is not to propose or justify specific models of data but rather to relate properties of known learning systems (and corresponding properties of data) to subpopulation coupling. Importantly, we aim to demonstrate that the coupling emerges from the interaction between the algorithm and the geometric properties of the data distribution and not from any explicit knowledge of subpopulations. 

\paragraph{Local algorithms:}
A simple example of a class of algorithms that will exhibit subpopulation coupling is $k$-NN-like algorithms and other algorithms that are in some sense locally smooth. If subpopulations are sufficiently ``clustered" so that including the example $(x,y)$ in the predictor will affect the prediction in the neighborhood of $x$ and the total weight of affected neighborhood is some fraction $\lambda_1$ of the subpopulation, then we will obtain subpopulation coupling with $\lambda_1$. In the more concrete (and extreme case), when for every point $x \in X$, the most distant point in $X_x$ is closer than the closest point from the other subpopulations we will get that any example from a subpopulation will cause a $1$-NN classifier to predict in the same way over the entire subpopulation. In particular, it would make it $\Lambda$-subpopulation-coupled for $\Lambda = (1,\ldots,1)$.

\paragraph{Linear classifiers:}
A more interesting case to understand is that of linear classifiers and by extension SVMs and (in a limited sense) neural networks. We will examine a high-dimensional setting, where $d \gg n$. We will assume that points within each subpopulation are likely to have relatively large inner product whereas for every subpopulation most points will, with high probability have, a substantially large component that is orthogonal to the span of $n$ random samples from other populations. These conditions are impossible to satisfy when $d \leq n$ but are easy to satisfy when $d$ is sufficiently large.
Formally, we assume that points in most datasets sampled from the data distribution satisfy the following condition:
\begin{defn}
\label{def:independent}
Let $X\subset \reals^d$ be a domain partitioned into subdomains $\{X_i\}_{i\in [N]}$. We say that a sequence of points $V=(x_1,\ldots,x_n)$ is $(\tau,\theta)$-independent if it holds that
\begin{itemize}
\item for all $i,j$ such that $x_i,x_j \in X_t$ for some $t$,  $ \la x_i,x_j \ra \geq \tau \|x_i\|_2 \|x_j\|_2$ and
\item for all $i$ such that $x_i \in X_t$, and any $v\in \spn(V\setminus X_t)$, $|\la x_i,v \ra| \leq \theta \|x\|_2 \|v\|_2$.
\end{itemize}
\end{defn}

We consider the performance of linear classifiers that approximately maximize the margin. Here, by ``approximately" we will simply assume that they output classifiers that achieve at least $1/2$ of the optimal margin achievable when separating the same points in the given dataset. Note that algorithms with this property are easy to implement efficiently via SGD on the cross-entropy loss \citep{soudry2018implicit} and also via simple regularization of the Perceptron algorithm \citep{shalev2005new}. We will also assume that the linear classifiers output by the algorithm lie in the span of the points in the dataset\footnote{A linear classifier can always be projected to the span of the points without affecting the margins.  This assumption allows us to avoid having to separately deal with spurious correlations between unseen parts of subpopulations and the produced classifiers.} Formally, we define approximately margin-maximizing algorithms in this multi-class setting (for convenience, restricted to the homogeneous case) as follows:
\begin{defn}
\label{def:max-margin}
An algorithm $\A$ is an {\em approximately margin maximizing $m$-class linear classifier} if given a
 dataset $S=((x_1,y_1),\ldots,(x_n,y_n)) \in (X\times [m])^n$ it outputs $m$ linear classifiers $w_1,\ldots,w_m$ satisfying:
 \begin{itemize}
 \item for every $k\in [m]$, $w_k$ lies in the span of $x_1,\ldots,x_n$;
 \item for every $x$, the prediction of $\A$ on $x$ depends only on the predictions of the classifiers $\sgn(\la x,w_k\ra)$ and;
 \item for every $k\in [m]$, let $V_- \dfn \{x \in X_S \cond \la x,w_k\ra < 0\}$ and $V_+ \dfn \{x \in X_S \cond \la x,w_k\ra \geq 0\}$. If $V_-$ can be linearly separated from $V_+$ by a homogeneous linear separator with margin $\gamma_k$ then for all $x \in X_S$, $| \la x,w_k\ra| \geq  \frac{\gamma_k}{2} \|x\|_2$.
 \end{itemize}
\end{defn}

We now show that linear classifiers over distributions that produce datasets independent in the sense of Definition \ref{def:independent} will have high subpopulation coupling. In order to guarantee strong coupling, we will assume that the set $V$ of points in a random dataset together with the set of points $V'$ that consists of additional samples from every mixture present in $V$ (namely, $V' \sim \prod_{j\in [N]_{S\#1}} M_j$) satisfy the independence condition with high probability. Formally, we establish the following result (the proof can be found in Appendix~\ref{app:linear-proof}).
\begin{thm}
\label{thm:linear-coupling}
Let $X\subset \reals^d$ be a domain partitioned into sub-domains $\{X_i\}_{i\in [N]}$ with subpopulations $\{M_i\}_{i\in [N]}$ over the sub-domains. Let $\A$ be any approximately margin maximizing $m$-class linear classifier and $\prior$ be a frequency prior. Assume that for $D\sim \D_\prior^{[N]}$ and $V\sim M_D^n$, $V' \sim \prod_{j\in [N]_{S\#1}} M_j$, with probability at least $1-\delta^2$,
$V \cup V'$ is $(\tau,\tau^2/(8\sqrt{n}))$-independent for some $\tau \in (0,1/2]$. Then for any labeling prior $\F$, $\A$ is $\Lambda$-subpopulation-coupled with probability $1-\delta$ and $\lambda_1 \geq 1-\delta$.
\end{thm}

As a simple example of subpopulations that will produce sets of points that are $(\tau,\tau^2/(8\sqrt{n}))$-independent with high probability  we pick each $M_i$ to be a spherically-symmetric distribution supported on a ball of radius $1$ around some center $z_i$ of norm $1$. We also pick the centers randomly and independently from the uniform distribution on the unit sphere.
It is not hard to see that, by the standard concentration properties of spherically-symmetric distributions, a set $V$ of $t$ samples from an arbitrary mixture of such distributions will be $(\tau,\theta)$-independent with high probability for $\tau \geq 1/2 - o(1)$ and $\theta = \tilde O(\sqrt{t/d})$. Thus for $t <2n$,  $d= \tilde O(n^2)$ suffices to ensure that $\theta \leq \tau^2/(8\sqrt{n})$.

%% file: interpolate-privacy.tex
\section{The Memorization, Privacy and Stability}
\label{sec:memprivacy}
So far we have discussed memorization by learning algorithms informally. In this section we give a simple definition of label memorization and demonstrate that fitting the training data in the setting we consider requires label memorization whenever there is enough (statistical or computational) uncertainty in the labels.  This allows us to show that limits on the memorization ability of an algorithm translate into a loss of accuracy (on long-tailed distributions). This result explains a recent empirical finding \citep{bagdasaryan2019differential,hooker2020compressed,HookerMCBD20} that in a dataset that is a mixture of several groups the loss in accuracy due to limited memorization will be higher on less frequent subgroups. Finally, we show that (even relatively weak forms of) differential privacy imply that the algorithm cannot memorize well.

To keep the notation cleaner we will discuss these results in the context of our simpler model from Sec.\ref{sec:main} but they can be easily adapted to our mixture model setting. For simplicity of notation, we will also focus on memorization of singleton elements.

\subsection{Memorization}
\label{sec:memory}
To measure the ability of an algorithm $\A$ to memorize labels we will look at how much the labeled example $(x,y)$ affects the prediction of the model on $x$. This notion will be defined per specific dataset and example but in our applications we will use the expectation of this value when the dataset is drawn randomly.
\begin{defn}
\label{def:memory}
 For a dataset $S=(x_i,y_i)_{i\in [n]}$ and $i \in [n]$ define
 $$\mem(\A,S,i) \dfn \pr_{h\sim \A(S)}[h(x_i) = y_i] - \pr_{h\sim \A(S^{\setminus i})}[h(x_i) = y_i] ,$$
 where $S^{\setminus i}$ denotes the dataset that is $S$ with $(x_i,y_i)$ removed.
\end{defn}
In this definition we measure the effect simply as the total variation distance\footnote{Strictly speaking, the memorization value can be negative (in which case it is equal to the negation of the TV distance) but for most practical algorithms we expect this value to be non-negative.} between the distributions of the indicator of the label being $y$, but other notions of distance could be appropriate in other applications. For this notion of distance our definition of memorization is closely related to the leave-one-out stability of the algorithm (see eq.~\eqref{eq:loo-stab}). Indeed, it is easy to see from this definition that LOO stability upper bounds the expected memorization:
$$ \fr{n} \E_{S\sim P^n} \lb \sum_{i \in [n]} \mem(\A,S,i) \rb  \leq \mathtt{LOOstab}(P,\A).$$ As in the case of stability label memorization can be related to the generalization gap in the following way (the proof follows immediately from taking the expectation over $S$).
\begin{lem}
\label{lem:mem2gap}
For every distribution $P$ over $X\times Y$ and any learning algorithm $\A$ we have that
\[ 
\fr{n} \E_{S\sim P^n} \lb \sum_{i\in [n} \mem(\A,S,i) \rb = \E_{S\sim P^n} \lb \err_S(\A,S) \rb  - \E_{S'\sim P^{n-1}} \lb \err_P(\A,S') \rb,
\]
where $\err_S(\A,S)$ is the expected empirical error of $\A$ on $S$: \[\err_S(\A,S) \dfn \fr{n} \sum_{i\in [n]}\pr_{h\sim \A(S)} [h(x_i)\neq y_i] .\]
\end{lem}
Note that the term $\E_{S'\sim P^{n-1}} \lb \err_P(\A,S') \rb$ is not exactly equal to the expectation of the generalization error $\E_{S\sim P^n} \lb \err_P(\A,S) \rb$, but in practice the difference between those is typically negligible (less than $1/n$). The immediate implication of Lemma \ref{lem:mem2gap} is that a large generalization gap indicates that many labels are memorized and vice versa.


An immediate corollary of our definition of memorization is that if $\A$ cannot predict the label $y_i$ of $x_i$ without observing it then it needs to memorize it to fit it. More formally,
\begin{lem}
\label{lem:alg-fit-memo}
For every dataset $S \in (X\times Y)^n$, learning algorithm $\A$ and index $i\in [n]$,
$$ \pr_{h\sim \A(S)}[h(x_i) \neq y_i] =  \pr_{h\sim \A(S^{\setminus i})}[h(x_i) \neq y_i]  -  \mem(\A,S,i).$$
In particular,
$$ \errn_S(\A,1) = \sum_{i \in [n],\ x_i \in X_{S\#1}} \pr_{h\sim \A(S^{\setminus i})}[h(x_i) \neq y_i]  -  \mem(\A,S,i).$$
\end{lem}

There can be several reasons why an algorithm $\A$ cannot predict the label on $x_i$ without observing it. The simplest one is that if there is statistical uncertainty in the label. To measure the uncertainty in a distribution $\rho$ over labels we will simply use 1 minus the maximum probability of any specific label:
 $$\|\rho\|_\infty \dfn \max_{y\in Y} \rho(y).$$ Note that $1-\|\rho\|_\infty$ is exactly the error of the Bayes optimal predictor given that the posterior distribution on the label is $\rho$.

Significant statistical uncertainty conditioned on knowing all the other labeled examples exists only when the labeling prior has high entropy (such as being uniform over a class of functions of VC dimension larger than $n$). In practice, there might exist a relatively simple model that explains the data well yet the learning algorithm cannot find (or even approximate) this model due to computational limitations. This can be modeled by considering the best accuracy in predicting the label of $x_i$ given $S^{\setminus i}$ for the restricted class of algorithms to which $\A$ belongs. For example, the uniform prior can be achieved for all polynomial-time algorithms by using a pseudo-random labeling function \cite{GoldreichGM86}. More generally, Lemma \ref{lem:alg-fit-memo} implies that any upper bound on the expected accuracy of a learning algorithm on an unseen singleton example implies the need to memorize the label in order to fit it. Thus the results in the remainder of this section extend directly to computational notions of uncertainty in place of $1-\|\rho\|_\infty$. We now spell out the properties of this simple statistical notion of uncertainty.
\begin{lem}
\label{lem:alg-forgets}
 Let $\rho$ be an arbitrary distribution over $Y$. For a dataset $S=(x_i,y_i)_{i\in [n]}$, $i \in [n]$ and $y \in Y$, let $S^{i\lar y}$ denote the dataset $S$ with $(x_i,y)$ in place of example $(x_i,y_i)$. Then we have:
$$\pr_{y \sim \rho, h\sim \A(S^{i\lar y})} \lb h(x) \neq y \rb \geq 1-\|\rho\|_\infty - \E_{y \sim \rho} [\mem(\A,S^{i\lar y},i)].$$
In particular, for every distribution $D$ and labeling prior $\F$ that also generates the noisy labeling function $\tilde f$ for every $f$ (as in Sec.~\ref{sec:noise})
$$\E_{f\sim \F, S\sim (D,\tilde f)^n} \lb \errn_S(\A,1) \rb \geq \E_{f\sim \F, S\sim (D, \tilde f)^n} \lb
\sum_{i \in [n],\ x_i \in X_{S\#1}} 1-\|\F(x_i|S^{\setminus i}\|_\infty ) -  \mem(\A,S,i)  \rb,$$
where $\F(x_i|S^{\setminus i})$ denotes the conditional distribution over the label of $x_i$ after observing all the other examples:
$$\F(x_i|S^{\setminus i}) = \Dist_{f\sim \F, S\sim (D,\tilde f)^n}[f(x_i) \cond \forall j\neq i, f(x_j) = y_j] .$$
\end{lem}
\begin{proof}
By Definition~\ref{def:memory}, for every $y$,
$$\pr_{h\sim \A(S^{i\lar y})}[h(x) = y] = \pr_{h\sim \A(S^{\setminus i})}[h(x) = y] + \mem(\A,S^{i\lar y},i) .$$
Thus,
\alequn{\E_{y \sim \rho, h\sim \A(S^{i\lar y})} \lb h(x) = y \rb &= \pr_{y\sim \rho, h \sim \A(S^{\setminus i})}[h(x) = y] + \E_{y \sim \rho} [\mem(\A,S^{i\lar y},i)]\\
& \leq  \max_{y'\in Y}\pr_{y\sim \rho}[y' = y] + \E_{y \sim \rho} [\mem(\A,S^{i\lar y},i)],}
giving the first claim.

The second claim follows from the definition of $\errn_S(\A,1)$ and observing that an expectation is taken on $f\sim \F$ that ensures that for every point the error will be averaged over all labelings of the point according to conditional distribution of the corresponding label.
\eat{ 
We denote by $\cS$ the marginal distribution over $S$ when $f\sim \F, S\sim (D,f)^n$.
\alequn{
\E_{f\sim \F, S\sim (D,f)^n} \lb \errn_S(\A,1) \rb &= \E_{S\sim \cS}\lb \sum_{i, x_i \in X_{S\#1}} \pr_{f\sim \F(|S), h\sim \A(S)}[f(x_i) \neq h(x_i)]   \rb   \\
 &=  \E_{S\sim \cS}\lb \sum_{i, x_i \in X_{S\#1}} 1 - \pr_{y \sim \F(x_i|S^{\setminus i}), h \sim \A(S)}[f(x_i) \neq h(x_i)]   \rb
}
}
\end{proof}

\subsection{The cost of limited memorization}
We will now translate Lemma \ref{lem:alg-forgets} into bounds on the excess error of algorithms that cannot memorize the labels well. For this purpose we will use the following definition.
\begin{defn}
\label{def:memorization-bounded}
 We say that a learning algorithm $\A$ is $\gamma$-memorization limited if for all $S\in (X,Y)^n$ and all $i \in [n]$ we have $\mem(\A,S,i) \leq  \gamma$.
\end{defn}
Bounds on memorization ability result directly from a variety of techniques, such as implicit and explicit regularization and model compression. Somewhat simplistically, one can think of these techniques as minimizing the sum some notion of capacity scaled by a regularization parameter $\lambda$ and the empirical error. Fitting a label that is not predicted correctly based on the rest of the dataset typically requires increasing the capacity. Therefore a regularized algorithm will not fit the example if the increase in the capacity (scaled by $\lambda$) does outweigh the decrease in the empirical error. These decisions are randomized and thus correspond to a bounded probability that the algorithm will memorize a label.

Using the definitions and Lemma \ref{lem:alg-forgets}, we immediately obtain the following example corollary on the excess error of any $\gamma$-memorization limited algorithm.
\begin{cor}
\label{cor:forget-single}
In the setting of Thm.~\ref{thm:main-noise}, let $\A$ be any $\gamma$-memorization limited algorithm. Then
$$\aerr(\prior,\F,\A) \geq \opt(\prior,\F) + \tau_1 \cdot \E_{D\sim D_\prior^X, f\sim \F, S\sim (D,\tilde f)^n} \lb
\sum_{i \in [n],\ x_i \in X_{S\#1}} \conf(S,i,\F) \lp 1-\|\F(x_i|S^{\setminus i}\|_\infty -  \gamma \rp \rb .$$
\end{cor}
The bound in this corollary depends on the expectation of the uncertainty in the label $1-\|\F(x_i|S^{\setminus i}\|_\infty$. While, in general, this quantity might be hard to estimate it might be relatively easy to get a sufficiently strong upper bound. For example, if for $f\sim \F$ the labeling is uniform and $k$-wise independent for $k$ that upper-bounds the typical number of distinct points (or subpopulations in the general case) then, with high probability, it will hold that $\|\F(x_i|S^{\setminus i}\|_\infty = 1/|Y|$. 
As discussed in Section \ref{sec:tail-bounds}, for Zipf prior distribution and $N \geq n$, any $\gamma$-memorization limited algorithm with  $\gamma < 1-1/|Y|$ being a constant will have excess error of $\Omega(1)$. Equivalently, any algorithm that achieves the optimal generalization error will need to memorize $\Omega(n)$ labels. In particular, it will have a generalization gap of $\Omega(1)$. These conclusions hold even in the presence of random noise. Consider, for example, the random classification noise model in which $\tilde f$ is defined by replacing the correct label $f(x)$ with a random and uniformly chosen one with probability $1-\kappa$. For this model we will have that for singleton examples $\conf(S,i,\F) \geq \kappa$. Thus we obtain that even noisy labels need to be memorized as long as $\kappa = \Omega(1)$.

\subsection{Cost of privacy}
\label{sec:privacy}
Memorization of the training data can be undesirable in a variety of settings. For example, in the context of user data privacy, memorization is known to lead to ability to mount black-box membership inference attacks (that discover the presence of a specific data point in the dataset) \citep{ShokriSSS17,LongBG17,LongBWBW18,truex2018towards} as well as ability to extract planted secrets from language models \citep{carlini2019secret}.
The most common approaches toward defending such attacks are based on the notion of differential privacy \citep{DworkMNS:06} that are formally known to limit the probability of membership inference by requiring that the output distribution of the learning algorithm is not too sensitive to individual data points.
Despite significant recent progress in training deep learning networks with differential privacy, they still lag substantially behind the state-of-the-art results trained without differential privacy \citep{shokri2015privacy,abadi2016deep,PapernotAEGT16,WuLKCJN17,PapernotAEGT17,McMahan18}. While some of this lag is likely to be closed by improved techniques, our results imply that the some of this gap is inherent due to the data being long-tailed. More formally, we will show that the requirements differential privacy imply a lower bound on the value of $\errn$ (for simplicity just for $\ell=1$). We will prove that this limitation applies even to algorithms that satisfy a very weak form of privacy: label privacy for predictions. It protects only the privacy of the label as in \citep{ChaudhuriHsu:11} and also with respect to algorithms that only output a prediction on an (arbitrary) fixed point \citep{dwork2018privacy}. Formally, we define:
\begin{defn}
\label{def:prediction-privacy}
Let $\A$ be an algorithm that given a dataset $S\in (X\times Y)^n$ outputs a random predictor $h\colon X\to Y$. We say that $\A$ is {\em  $(\eps,\delta)$-differentially label-private prediction} algorithm if for every $x \in X$ and datasets $S$ that only differ in a label of a single element we have for any subset of labels $Y'$,
$$\pr_{h\sim \A(S)}[h(x) \in Y'] \leq e^\eps \cdot \pr_{h\sim \A(S')}[h(x) \in Y'] + \delta .$$
\end{defn}
It is easy to see that any algorithm that satisfies this notion of privacy is $(e^\eps - 1 + \delta)$-memorization limited. A slightly more careful analysis in this case gives the following analogues of Lemma \ref{lem:alg-forgets} and Corollary~\ref{cor:forget-single}.
\begin{thm}
\label{thm:privacy-forgets}
Let $\A$ be an $(\eps,\delta)$-differentially label-private prediction algorithm and let $\rho$ be an arbitrary distribution over $Y$. For a dataset $S=(x_i,y_i)_{i\in [n]}$, $i \in [n]$ and $y \in Y$, we have:
$$\pr_{y \sim \rho, h\sim \A(S^{i\lar y})} \lb h(x) = y \rb \leq e^\eps \cdot \|\rho\|_\infty + \delta .$$

In particular, in the setting of Thm.~\ref{thm:main-noise}, for every distribution $D$ and labeling prior $\F$,
$$\E_{f\sim \F, S\sim (D,\tilde f)^n} \lb \errn_S(\A,1) \rb \geq \E_{f\sim \F, S\sim (D,\tilde f)^n} \lb
\sum_{i \in [n],\ x_i \in X_{S\#1}} 1 -  e^\eps \cdot \| \F(x_i|S^{\setminus i}) \|_\infty - \delta  \rb. $$
and, consequently,
$$\aerr(\prior,\F,\A) \geq \opt(\prior,\F) + \tau_1 \cdot \E_{D\sim D_\prior^X, f\sim \F, S\sim (D,\tilde f)^n} \lb
\sum_{i \in [n],\ x_i \in X_{S\#1}} \conf(S,i,\F) \lp 1 -  e^\eps \cdot \| \F(x_i|S^{\setminus i}) \|_\infty -  \delta \rp \rb .$$
\end{thm}
\begin{proof}
By the definition of $(\eps,\delta)$-differential label privacy for predictions, for every $y$,
$$\pr_{h\sim \A(S^{i\lar y})}[h(x) = y] \leq e^\eps \cdot \pr_{h\sim \A(S)}[h(x) = y] + \delta .$$
Thus,
$$\pr_{y \sim \rho, h\sim \A(S^{i\lar y})} \lb h(x) = y \rb \leq  e^\eps \pr_{y \sim \rho, h\sim \A(S)} \lb h(x) = y \rb +\delta \leq e^\eps \|\rho\|_\infty + \delta .$$
The rest of the claim follows as before.
\end{proof}

This theorem is easy to extend to any subpopulation from which only $\ell$ examples have been observed using the group privacy property of differential privacy. This property implies that if $\ell$ labels are changed then the resulting distributions are $(\ell \eps, \ell e^{\ell-1} \delta)$-close (in the same sense) \citep{DworkRoth:14}. The total weight of subpopulations that have at most $\ell$ examples for a small value of $\ell$ is likely to be significant in most modern datasets. Thus this may formally explain at least some of the gap in the results currently achieved using differentially private training algorithms and those achievable without the privacy constraint.

\paragraph{Uniform stability:}
A related notion of stability is uniform prediction stability \citep{BousquettE02,dwork2018privacy} that, in the context of prediction,  requires that changing any point in the dataset does not change the label distribution on any point by more than $\gamma$ in total variation distance. This notion is useful in ensuring generalization \citep{BousquettE02,FeldmanV:19} and as a way to ensure robustness of predictions against data poisoning.
In this context, $\gamma$-uniform stability implies that the algorithm is $\gamma$-memorization limited (and also is $(0,\gamma)$-differentially private for predictions). Therefore Corollary \ref{cor:forget-single} implies limitations of such algorithms.

\subsection{Disparate effect of limited memorization}
\label{sec:disparate}
Corollary \ref{cor:forget-single} and Theorem \ref{thm:privacy-forgets} imply that limiting memorization increases the generalization error of an algorithm on long-tailed (and sufficiently hard) learning problems. Moreover, the excess error due to limited memorization depends on the prior $\prior$, hardness of the problem and the number of samples $n$. This implies that if the data distribution consists of several subgroups with different properties, then the cost of limiting memorization can be different for these subgroups. In particular, the cost can be higher for smaller subgroups or those with more distinct subpopulations. These are not hypothetical scenarios. For differential privacy these effects were observed in a concurrent work of \citet{BagdasaryanShmatikov19}. For model compression the differences in the costs have been confirmed and investigated in a subsequent work of Hooker et al.~\citep{hooker2020compressed,HookerMCBD20}. In addition to disparate effects, these works empirically demonstrate that the increase in error is most pronounced on atypical examples.

As a concrete example of why our long-tail theory explains the different costs we consider a $10$-class classification problem over $N=5,000$ subpopulations, Zipf prior $\pi$, and $n=50,000$ samples. We will also assume for simplicity, that the labeling prior is uniform and independent over all subpopulations and there is no noise. Let $\A$ be a $\gamma$-memorization limited learning algorithm for $\gamma =1/2$. The choice of $\gamma$ does not affect the comparison as it will scale the excess error for all subgroups in the same way. The labels of all the subpopulations that have not been observed in the sample are completely unpredictable and therefore the expected error of the optimal algorithm in this setting is equal to $$\opt(\pi,\F) = \lp 1 - \frac{1}{|Y|} \rp \sum_{j\in [N], \alpha = \piN(j)} \alpha \cdot (1-\alpha)^n \ .$$
To compute this value in our setting of parameters we will use $\pi$ instead of $\piN$ as those are very close for large $N$ and it is easier to perform (and verify) computations on $\pi$. This gives us $\opt(\pi,\F) \approx 0.018$. Applying Corollary \ref{cor:forget-single}, we obtain that cost of limiting memorization to $1/2$ is $\approx 0.015$.

Now, consider the same question but for a sample that only has $10,000$ examples. Then $\opt(\pi,\F) \approx 0.113$ and the cost of limited memorization $\approx 0.035$. Finally, consider the same question but with the number of subpopulations $N=25,000$ and $n=50,000$ (corresponding to a harder learning problem). Then $\opt(\pi,\F) \approx 0.107$ and the cost of limited memorization is $\approx 0.031$.

Next, assume that we are given a learning problem that is a mixture of the first and second settings, namely, the population is $P= \frac{5}{6} P_1 + \frac{1}{6} P_2$ and we are given $n=60,000$ examples. Then in each subgroup we still have the same optimums and the same cost of limited memorization. The cost of limited memorization is more than twice higher for the smaller subgroup in this mixture problem. Similarly, in the mixture of the first and third settings ($P= \frac{1}{2} P_1 + \frac{1}{2} P_3$ and $n=100,000$) the cost of limited memorization is twice higher for the subgroup with a harder prediction problem.

The cost of memorization with 10 classes and $\gamma = 0.5$ is the same as the cost of (label) differential privacy for predictions with $\eps = \ln 6$ and $\delta \approx 0$ so the same conclusions follow from Theorem \ref{thm:privacy-forgets}.

Understanding of the causes of such disparate effects can be used to design mitigation strategies. For example, by using different levels of regularization (or compression) on different subgroups the costs can be balanced. Similarly, a different privacy parameter can be used for different subgroups (assuming that the additional risk of privacy violations is justified by the increase in the accuracy).

%% file: interpolate-app.tex
\appendix
\remove{
\begin{figure}
  {\centering\includegraphics[width=0.8\linewidth]{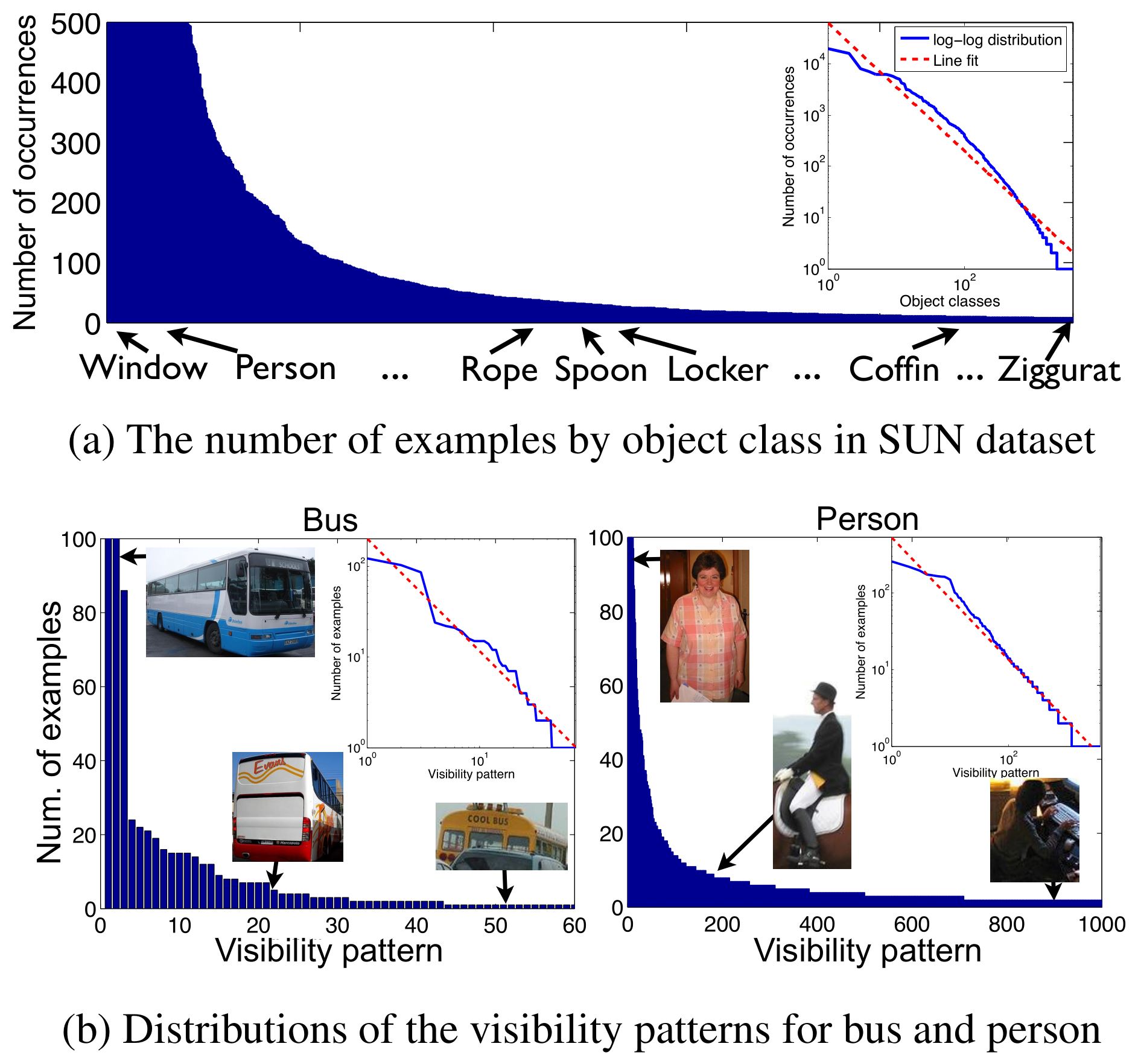}}
  \caption{Long tail of class frequencies and subpopulation frequencies within classes. The figure is taken from \citep{zhu2014capturing} with the authors' permission.}
  \label{fig:long}
\end{figure}

\begin{figure}
  \includegraphics[width=\linewidth]{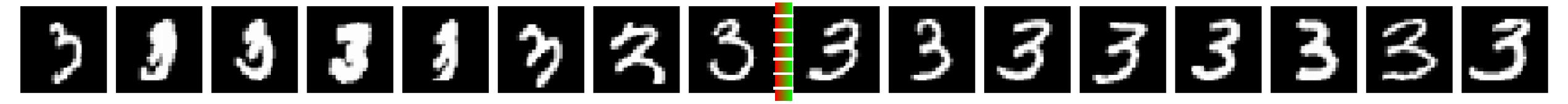}
  \includegraphics[width=\linewidth]{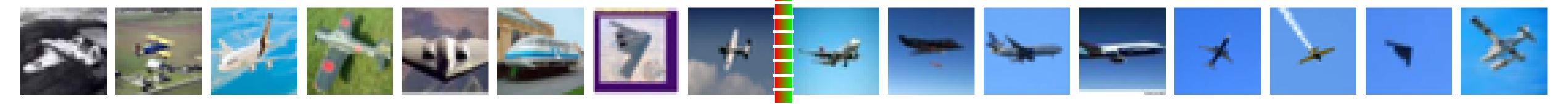}
  \caption{Hardest examples for a differentially private to predict accurately (among those accurately predicted by a non-private model) on the left vs the easiest ones on the right. Top row is for digit ``3" from the MNIST dataset and the bottom row is for the class ``plane" from the CIFAR-10 dataset. The figure is extracted from \citep{carlini2018prototypical} with the authors' permission. Details of the training process can be found in the original work.}
  \label{fig:prototypical}
\end{figure}
}

\section{Proof Lemma~\ref{lem:cond-density}}
\label{app:frequency}
The key property of our problem definition is that it allows to decompose the probability of a dataset (under the entire generative process) into a probability of seeing one of the points in the dataset and the probability of seeing the rest of the dataset under a similar generative process. Specifically, we prove the following lemma.
\begin{lem}
\label{lem:factorize}
For $x \in X$, a sequence of points $V = (x_1,\ldots,x_n)\in X^n$ that includes $x$ exactly $\ell$ times, let $V\setminus x$ be equal to $V$ with all the elements equal to $x$ omitted. Then for any frequency prior $\pi$ and $\alpha$ in the support of $\piN$, we have
$$\pr_{D\sim \D_\prior^X, U\sim D^n}[U = V \cond D(x) = \alpha] = \alpha^\ell \cdot (1-\alpha)^{n-\ell} \cdot \pr_{D'\sim \D_\prior^{X\setminus\{x\}}, U' \sim D^{n-\ell}}[U' = V\setminus x] .$$
In particular:
$$\pr_{D\sim \D_\prior^X, U\sim D^n}[U = V] = \E_{\alpha \sim \piN} \lb \alpha^\ell \cdot (1-\alpha)^{n-\ell}\rb \cdot  \pr_{D'\sim \D_\prior^{X\setminus\{x\}}, U' \sim D^{n-\ell}}[U' = V\setminus x].$$
\end{lem}
\begin{proof}
We consider the distribution of $D\sim \D_\prior^X$ conditioned on $D(x) = \alpha$ (which, by our assumption, is an event with positive probability). We denote this distribution by $\D_\prior^X(|D(x)=\alpha)$. From the definition of $\D_\prior^X$ we get that a random sample $D$ from $\D_\prior^X(|D(x)=\alpha)$ can generated by setting $D(x)=\alpha$, then for all $z\in X\setminus \{x\}$, sampling $p_z$ from $\pi$ and normalizing the results to sum to $1-\alpha$. That is, defining $$D(z) = (1-\alpha) \frac{p_z}{\sum_{z\in X\setminus \{x\}} p_z} .$$ From here we obtain that an equivalent way to generate a random sample from $\D_\prior^X(|D(x)=\alpha)$ is to sample $D'$ from $\D_\prior^{X\setminus\{x\}}$ and then multiply the resulting p.m.f.~by $1-\alpha$ (with $D(x) = \alpha$ as before). Naturally, for any $D$,
$$\pr_{U\sim D^n}[U=V] = \prod_{i\in [n]} D(x_i) .$$
Now we denote by $I_{-x}$ the subset of indices of elements of $V$ that are different from $x$: $I_x = \{i \in [n]\cond x_i \neq x\}$. We can now conclude:
\alequn{\pr_{D\sim \D_\prior^X, U\sim D^n}[U = V \cond D(x) = \alpha] &= \pr_{D\sim \D_\prior^X(|D(x)=\alpha), U\sim D^n}[U = V ] \\
&= \E_{D\sim \D_\prior^X(|D(x)=\alpha)} \lb  \prod_{i\in [n]} D(x_i) \rb  \\
&= \E_{D\sim \D_\prior^X(|D(x)=\alpha)} \lb  \alpha^\ell \prod_{i \in I_{-x}} D(x_i) \rb \\
&= \alpha^\ell \cdot (1-\alpha)^{n-\ell} \cdot \E_{D\sim \D_\prior^X(|D(x)=\alpha)} \lb  \prod_{i \in I_{-x}} \frac{D(x_i)}{1-\alpha} \rb \\
&= \alpha^\ell \cdot (1-\alpha)^{n-\ell} \cdot \E_{D'\sim \D_\prior^{X\setminus\{x\}}} \lb  \prod_{i \in I_{-x}} D'(x_i) \rb \\
&= \alpha^\ell \cdot (1-\alpha)^{n-\ell} \cdot \pr_{D'\sim \D_\prior^{X\setminus\{x\}}, U' \sim D^{n-\ell}}[U' = V\setminus x] .}
The second part of the claim follows directly from the fact that, by definition of $\piN$,
$$\pr_{D\sim \D_\prior^X, U\sim D^n}[D(x) = \alpha] = \piN(\alpha) .$$
\end{proof}

We can now prove Lemma \ref{lem:cond-density} which we restate here for convenience.
\begin{lem}[Lemma \ref{lem:cond-density} restated]
\label{lem:cond-density-2}
For any frequency prior $\pi$, $x \in X$ and a sequence of points $V = (x_1,\ldots,x_n)\in X^n$ that includes $x$ exactly $\ell$ times, we have
$$\E_{D \sim \D^X_\prior, U\sim D^n}[D(x) \cond U= V] =  \frac{\E_{\alpha \sim \piN} \lb \alpha^{\ell+1} \cdot (1-\alpha)^{n-\ell}\rb}{\E_{\alpha \sim \piN} \lb \alpha^\ell \cdot (1-\alpha)^{n-\ell}\rb }  .$$
\end{lem}
\begin{proof}
We first observe that by the Bayes rule and Lemma \ref{lem:factorize}:
\alequn{
\pr_{D \sim \D^N_\prior, U\sim D^n}[D(x) = \alpha \cond U= V] &= \frac{\pr_{D \sim \D^N_\prior, U\sim D^n}[U= V \cond D(x) = \alpha] \cdot \pr_{D \sim \D^N_\prior, U\sim D^n}[D(x) = \alpha]}{\pr_{D \sim \D^N_\prior, U\sim D^n}[U= V]}\\
& = \frac{\alpha^\ell \cdot (1-\alpha)^{n-\ell} \cdot \pr_{D'\sim \D_\prior^{X\setminus\{x\}}, U' \sim D^{n-\ell}}[U' = V\setminus x] \cdot \piN(\alpha)}{\E_{\beta \sim \piN} \lb \beta^\ell \cdot (1-\beta)^{n-\ell}\rb \cdot  \pr_{D'\sim \D_\prior^{X\setminus\{x\}}, U' \sim D^{n-\ell}}[U' = V\setminus x]}\\
& = \frac{\alpha^\ell \cdot (1-\alpha)^{n-\ell} \cdot \piN(\alpha)}{\E_{\beta \sim \piN} \lb \beta^\ell \cdot (1-\beta)^{n-\ell}\rb } .
}
This leads to the claim:
\alequn{ \E_{D \sim \D^N_\prior, U\sim D^n}[D(x) \cond U= V] &= \sum_{\alpha \in \supp(\piN)} \alpha \cdot \pr_{D \sim \D^N_\prior, U\sim D^n}[D(x) = \alpha \cond U= V] \\
&= \sum_{\alpha \in \supp(\piN)} \alpha \cdot \frac{\alpha^\ell \cdot (1-\alpha)^{n-\ell} \cdot \piN(\alpha)}{\E_{\beta \sim \piN} \lb \beta^\ell \cdot (1-\beta)^{n-\ell}\rb} \\
&= \frac{\E_{\alpha \sim \piN} \lb \alpha^{\ell+1} \cdot (1-\alpha)^{n-\ell}\rb}{\E_{\alpha \sim \piN} \lb \alpha^\ell \cdot (1-\alpha)^{n-\ell}\rb }.
}
\end{proof}

\section{Proof of Theorem~\ref{thm:linear-coupling}}
\begin{thm}[Thm.~\ref{thm:linear-coupling} restated]
Let $X\subset \reals^d$ be a domain partitioned into sub-domains $\{X_i\}_{i\in [N]}$ with subpopulations $\{M_i\}_{i\in [N]}$ over the sub-domains. Let $\A$ be any approximately margin maximizing $m$-class linear classifier and $\prior$ be a frequency prior. Assume that for $D\sim \D_\prior^{[N]}$ and $V\sim M_D^n$, $V' \sim \prod_{j\in [N]_{S=1}} M_j$, with probability at least $1-\delta^2$,
$V \cup V'$ is $(\tau,\tau^2/(8\sqrt{n}))$-independent for some $\tau \in (0,1/2]$. Then for any labeling prior $\F$, $\A$ is $\Lambda$-subpopulation-coupled with probability $1-\delta$ and $\lambda_1 \geq 1-\delta$.
\end{thm}


\label{app:linear-proof}
\begin{proof}
For the given priors $\prior$ and $\F$, let $S=((x_1,y_1),\ldots,(x_n,y_n))$ be a dataset sampled from $(M_D,L_f)^n$ for $D\sim \D_\prior^{[N]}$ and $f\sim \F$. Let $V = (x_1,\ldots,x_n)$. Let $T \dfn [N]_{S=1}$ and let $V' = (x'_j)_{j\in K}$ be sampled from $\prod_{j\in T} M_j$, that is, $V'$ consists of additional independent samples from every subpopulation with a single sample.

We will show that for any $V\cup V'$ that is $(\tau,\theta \dfn \tau^2/(8\sqrt{n}))$-independent, the output $w_1,\ldots,w_m$ of any approximately  margin maximizing $m$-class linear classifier $\A$ gives predictions on $V'$ that are consistent with those on $V$ (which are defined by $S$):
if $x_i \in X_t$ for $t\in T$ then for every $k \in [m]$, $$\sgn(\la w_k, x'_t \ra) = \sgn(\la w_k, x_i \ra) .$$
By Defn.~\ref{def:max-margin}, this implies that the prediction of the classifier on $x'_t$ is identical to that on $x_i$. By our assumption, $V \cup V'$ is {\em not} $(\tau,\tau/(4\sqrt{n}))$-independent with probability at most $\delta^2$. By Markov's inequality, probability over the choice of $V$ such that, the probability over the choice of $V'$ that $V \cup V'$ is not $(\tau,\theta)$-independent is more than $\delta$, is at most $\delta$. By our definition of $V'$, the marginal distribution of $x'_t$ is exactly $M_t$. This implies that, with probability at least $1-\delta$ over the choice of the dataset $S$, for every $x \in X_{S=1}$, and $x'\sim M_x$ we have
$$\TV\lp \Dist_{h\sim \A(S)} [h(x)], \Dist_{x' \sim M_x, h\sim \A(S)} [h(x')] \rp \leq \delta $$
as required by Defn.~\ref{def:subpop-coupled} (for $\ell =1$).

To prove the stated consistency property for $V\cup V'$ that is $(\tau,\theta)$-independent, we will first show that every subset of points in $V$ can be separated from its complement with margin $\gamma$ of $\Omega(1/\sqrt{n})$. We will then use the properties of approximately margin maximizing classifiers and, again, independence to obtain consistency.

For any vector $v$, we denote $\bar{v} \dfn v/\|v\|_2$. To show that the margin is large, we define the weights explicitly by using one representative point from every subpopulation in $V$. Without loss of generality, we can assume that these representatives are $x_1,\ldots,x_r$ for some $r \leq n$. Let $z_1,\ldots,z_r \in \pmi$ be an arbitrary partition of these representatives into positively and negatively labeled ones. We define $w \dfn \sum_{j \in [r]} z_j \bar{x}_j$ and consider the linear separator given by $\bar w$.

To evaluate the margin we first observe that $\|w\|_2 \leq \sqrt{2r}$. This follows via induction on $r$:
\alequn{\left\| \sum_{j \in [r]} z_j \bar x_j \right\|_2^2 &=  \left\| \sum_{j \in [r-1]} z_j \bar x_j \right\|_2^2 + \left\|\bar x_j \right\|_2^2 + 2 z_r \left\la \sum_{j \in [r-1]} z_j \bar x_j, \bar x_r \right\ra
\\ &\leq 2(r-1) + 1 + 2 \frac{\tau^2}{8\sqrt n} \left\| \sum_{j \in [r-1]} z_j \bar x_j \right\|_2   \\
&\leq \frac{4(r-1)}{3} + 1 + \frac{1}{16\sqrt n} \cdot \sqrt{2(r-1)} \leq 2r .}

Now for $i \in [n]$, assume that $x_i \in X_t$ and (without loss of generality) that $x_r$ is the representative of subdomain $X_t$.
Then \alequn{z_r \la \bar x_i, \bar w \ra &= \fr{\|w\|_2} \lp \la \bar x_i,  \bar x_r \ra +  z_r\left\la \bar x_i, \sum_{j \in [r-1]} z_j \bar x_j \right\ra \rp \\ & \geq  \fr{\|w\|_2} \lp  \tau - \frac{\tau^2}{8\sqrt{n}} \left\| \sum_{j \in [r-1]} z_j \bar x_j \right\|_2   \rp \\ & \geq \frac{\tau}{\|w\|_2} \lp 1 - \frac{\tau \sqrt{2(r-1)}}{8\sqrt{n}}\rp \geq \frac{ \tau }{2 \sqrt{n}} .
 }

Thus we obtain that $x_i$ is labeled in the same way as its representative $x_r$ and with margin of at least $\frac{\tau}{2\sqrt{n}}$. This holds for all $i\in [n]$ and therefore $\bar{w}$ shows that the desired separation can be achieved with margin of at least $\frac{\tau}{2\sqrt{n}}$.

Let $w_1,\ldots,w_k$ be the linear separators returned by $\A$. Let $w$ be one of them. By our assumptions on $\A$, $w$ separates $V$ with margin of at least $\gamma \dfn \frac{\tau}{4\sqrt{n}}$ and further it lies in the span on $V$. Namely, there exist $\alpha_1,\ldots,\alpha_n$ such that $w = \sum_{i\in[n]} \alpha_i \bar x_i$.

We now pick an arbitrary singleton point from $V$. Without loss of generality we assume that it is $x_n$, $\la x_n, w\ra \geq \gamma \|x_n\|_2$ and let $x \in V'$ be the point from the same subdomain $X_t$. Let $v \dfn \sum_{i\in[n-1]} \alpha_i \bar x_i$ be the part of $w$ that excludes $x_n$. By our assumption, $x_n$ is a singleton and therefore the points in $(x_1,\ldots,x_{n-1})$ are from other subdomains. By the independence of $V$, this implies that $|\la \bar x_n, v \ra| \leq \theta \|v\|_2$ and
$|\la \bar{x}, v \ra| \leq \theta \|v\|_2$.

Now we need to show that the margin condition implies that $\alpha_n$ is sufficiently large.
Specifically,
$$\gamma \leq \la \bar x_n, w  \ra = \alpha_n + \la \bar x_n, v \ra \leq \alpha_n + \theta \|v\|_2,$$
and thus $$\alpha_n \geq \gamma - \theta \|v\|_2 \geq \gamma - \theta (1+\alpha_n), $$
where we used the fact that, by the triangle inequality, $\|v\|_2 \leq \|w\|_2 + \|\alpha_n \bar x_n \|_2 \leq 1+\alpha_n$.
This implies that $\alpha_n \geq \frac{\gamma -\theta}{1+\theta}$. We can now bound $\la \bar x,w\ra$
\alequn{
\la \bar x,w\ra &= \la \alpha_n \bar x,\bar x_n\ra + \la \bar x ,v \ra  \geq \alpha_n \tau - \theta \|v\|_2 \geq \alpha_n \tau - \theta (1+\alpha_n) = \alpha_n (\tau -\theta) - \theta\\
& \geq \frac{(\gamma -\theta)(\tau -\theta)}{1+\theta} - \theta \geq \frac{\lp \frac{\tau}{4\sqrt{n}} - \frac{\tau^2}{8\sqrt{n}} \rp  \lp \tau - \frac{\tau^2}{8\sqrt{n}}\rp}{1+\frac{\tau^2}{8\sqrt{n}}} - \frac{\tau^2}{8\sqrt{n}} > 0.
}
where the last inequality assumes that $n\geq 4$. Thus we obtain that for every $w \in \{w_1,\ldots,w_m\}$, every point in $V' \cap X_{S=1}$ will be classified by $w$ in the same way as the point from the same subpopulation in $S$.
\end{proof}
